\documentclass{article}

\usepackage[verbose=true,letterpaper]{geometry}
\AtBeginDocument{
  \newgeometry{
    textheight=9in,
    textwidth=5.5in,
    top=1in,
    headheight=12pt,
    headsep=25pt,
    footskip=30pt
  }
}

\usepackage[utf8]{inputenc} 
\usepackage[T1]{fontenc}    
\usepackage{hyperref}       
\usepackage{url}            
\usepackage{booktabs}       
\usepackage{amsfonts}       
\usepackage{nicefrac}       
\usepackage{microtype}      
\usepackage{xcolor}         

\usepackage{bm} 

\usepackage{amsmath}
\usepackage{amsthm} 
\usepackage{bbm} 
\usepackage{xparse}
\usepackage{xfrac}	
\usepackage{enumitem}	
\usepackage{tikz}

\usepackage{natbib}
\setcitestyle{numbers,square}

\makeatletter
\newcommand*{\centernot}{%
	\mathpalette\@centernot
}
\def\@centernot#1#2{%
	\mathrel{%
		\rlap{%
			\settowidth\dimen@{$\m@th#1{#2}$}%
			\kern.5\dimen@
			\settowidth\dimen@{$\m@th#1=$}%
			\kern-.5\dimen@
			$\m@th#1\not$%
		}%
		{#2}%
	}%
}
\makeatother

\NewDocumentCommand{\independent}{}{\perp\mkern-9.5mu\perp}
\NewDocumentCommand{\dependent}{}{\centernot{\independent}}
\NewDocumentCommand{\notindependent}{}{\dependent}
\NewDocumentCommand{\onehalf}{}{\sfrac{1}{2}}
\newcommand{\BreakingSmallSpace}{\hspace{.16667em}} 
\newcommand{\Slash}{\,/\BreakingSmallSpace}
\newcommand{\asswlog}{w.\,l.\,o.\,g.\ }

\NewDocumentCommand{\eg}{}{e.\,g.\ }
\NewDocumentCommand{\Eg}{}{E.\,g.\ }
\NewDocumentCommand{\ie}{}{i.\,e.\ }
\NewDocumentCommand{\Ie}{}{I.\,e.\ }

\NewDocumentCommand{\txt}{m}{\text{\normalfont{#1}}}
\NewDocumentCommand{\val}{m}{\mathcal{#1}}
\NewDocumentCommand{\gPhys}{m}{#1^{\txt{phys}}}
\NewDocumentCommand{\gDescr}{m}{#1^{\txt{descr}}}
\NewDocumentCommand{\gTransf}{m}{#1^{\txt{phys}}}
\NewDocumentCommand{\gVisible}{m}{#1^{\txt{visible}}}
\NewDocumentCommand{\gMask}{m}{#1^{\txt{mask}}}
\NewDocumentCommand{\gUnion}{m}{#1^{\txt{union}}}
\NewDocumentCommand{\gUnionId}{m}{#1^{\txt{union}}_{\txt{detect}}}
\NewDocumentCommand{\gCF}{m}{#1^{\txt{CF}}}
\NewDocumentCommand{\gIdent}{m}{#1^{\txt{ident}}}
\NewDocumentCommand{\gDetect}{m}{#1^{\txt{detect}}}
\NewDocumentCommand{\FixedR}{}{_{R=r}}

\definecolor{boldcolor}{gray}{0.18}
\NewDocumentCommand{\defName}{m}{\textbf{\textcolor{boldcolor}{#1}}}

\DeclareMathOperator{\PearlDo}{do}

\DeclareMathOperator{\Pa}{Pa}
\DeclareMathOperator{\pa}{pa}
\DeclareMathOperator{\Anc}{Anc}

\DeclareMathOperator{\Adj}{Adj}
\DeclareMathOperator{\Acycl}{Acycl}
\DeclareMathOperator{\supp}{supp}
\DeclareMathOperator{\const}{const}
\DeclareMathOperator{\xor}{xor}

\DeclareMathOperator{\boolAnd}{and}
\DeclareMathOperator{\biasSource}{Source}

\theoremstyle{plain} 
\newtheorem{lemma}{Lemma}[section]

\newtheorem{prop}[lemma]{Proposition}
\newtheorem{cor}[lemma]{Corollary}

\theoremstyle{definition} 
\newtheorem{definition}[lemma]{Definition}
\newtheorem{assumption}[lemma]{Assumption}

\newtheorem{example}[lemma]{Example}
\newtheorem{rmk}[lemma]{Remark}

\colorlet{color_f}{black}
\colorlet{color_g}{black!10!orange}
\colorlet{color_p0}{gray!30!blue}
\colorlet{color_p1}{black!30!green}
\newcommand{\plotlinewidth}[0]{0.15em}

\title{
Causal Modeling in Multi-Context Systems: Distinguishing Multiple Context-Specific Causal Graphs which Account for Observational Support}

\author{%
    Martin Rabel$^{1,2,3}$,
    Wiebke Günther$^{1,4}$,\\
    Jakob Runge$^{*,2,3,4,1}$,
    Andreas Gerhardus$^{*,1}$\\
    \small
    $^*$equal supervision\\
    \small
    $^1$German Aerospace Center (DLR), Institute of Data Science, Jena, Germany\\
    \small$^2$Center for Scalable Data Analytics and Artificial Intelligence\\[-0.3em]
    \small (ScaDS.AI) Dresden\Slash{}Leipzig, Germany \\
    \small$^3$Technische Universität Dresden, Faculty of Computer Science, Dresden, Germany \\
    \small$^4$Technische Universität Berlin, Institute of Computer Engineering and Microelectronics,\\[-0.3em]
    \small Berlin, Germany%
}

\begin{document}

\maketitle

\begin{abstract}
    Causal structure learning with data from multiple contexts carries both opportunities and challenges. Opportunities arise from considering shared and context-specific causal graphs enabling to generalize and transfer causal knowledge across contexts. However, a challenge that is currently understudied in the literature is the impact of differing observational support between contexts on the identifiability of causal graphs. Here we study in detail recently introduced \citep{MethodPaper} causal graph objects that capture both causal mechanisms and data support, allowing for the analysis of a larger class of context-specific changes, characterizing distribution shifts more precisely. We thereby extend results on the identifiability of context-specific causal structures and propose a framework to model context-specific independence (CSI) within structural causal models (SCMs) in a refined way that allows to explore scenarios where these graph objects differ. We demonstrate how this framework can help explaining phenomena like anomalies or extreme events, where causal mechanisms change or appear to change under different conditions. Our results contribute to the theoretical foundations for understanding causal relations in multi-context systems, with implications for generalization, transfer learning, and anomaly detection. Future work may extend this approach to more complex data types, such as time-series.
\end{abstract}

\section{Introduction}\label{sec:intro}
The combination of data from multiple datasets obtained from similar generating processes (contexts) is an important topic of study. An example are climate variables such as temperature and precipitation in different contexts such as humid and dry climate regions.
A key task then is to understand which changes in the underlying model may explain observed differences.
Such understanding is relevant to reason about the validity  of transfer of knowledge between contexts; this in turn is especially relevant to causal models \citep{SelectionVars, JCI} that aim to be robust \citep{RojasCarulla2018,AnchorRegression} across contexts.
Data from multiple contexts has both shared (between contexts) and individual
(per context) properties. The physical mechanisms in humid and dry climates are often the same, but the observable range of values differs, \eg, in that no freezing temperatures may be observed in a (hot and) dry region.

In order to capture as much
information about the underlying system(s) as possible,
it seems natural to consider understanding
qualitative aspects, for example, causal graphs, of both
the shared and the individual contexts \citep{LDAG_logical}.
We focus on representing qualitative information about the individual contexts, enriched with information from the joint model.
The relevance of available observations indicated above becomes apparent from following problem:
One cannot infer properties of mechanisms outside the range of values
that are actually observed (observational support), without prior knowledge about extrapolation.
But when combining data from multiple contexts,
the other contexts do provide knowledge about extrapolation
for an individual context.
Indeed it turns out that
combining support-properties and causal dependencies in a
single graphical object \citep{MethodPaper} allows for qualitative statements
by tracking few qualitative properties; a quantitative tracking of support
is usually impractical on continuous data.
Distinctions based on availability of data per context,
have interesting implications \eg for understanding
anomalies or extreme events.
It provides a possible explanation why
it seems to often be the case that (presumably robust) causal mechanisms
apparently change under extreme conditions (§\ref{sec:applications}). 

Intuitively, per-context information, from the
causal model perspective,
should be a qualitative change.
For example $Y = \mathbbm{1}(R) \times X + \eta_Y$.
Such a structure induces a
context-specific independence (CSI), \eg $X \independent Y | R=0$.
One may be tempted presume this implication to be an equivalence, and thus to attribute the occurrence of CSI to such a changing mechanism in the model.
Intriguingly, this direction -- drawing conclusions from CSI-structure
\citep{LDAG_definition, LDAG_logical, LDAG_learning}
about SCM-structure -- is more subtle, as the example below
(extending on observations of \citep{MethodPaper}) illustrates%
\footnote{We do not discuss finite-sample properties,
but these effects also occur,
\eg if observational support becomes narrow on the source compared to the derivative
of the mechanism and noise on the target (Rmk. \ref{rmk:replace_support_by_finitesample}).}:

\begin{example}\label{example:intro}
	Context specific independence from non-observation:\\
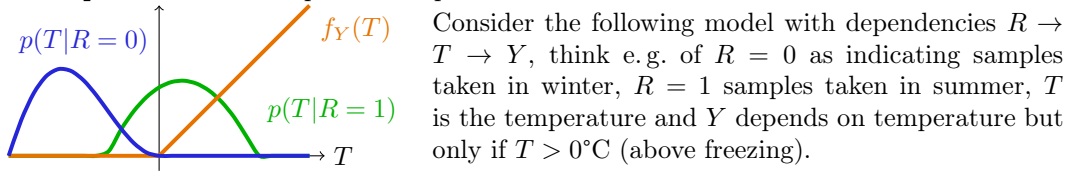
	\begin{minipage}{0.4\textwidth}
		\begin{tikzpicture}[domain=-2:2]
		\draw[->] (-2,0) -- (2.2,0) node[right] {$T$};
		\draw[->] (0,-0.2) -- (0,2);
		
		\draw[color=color_p1, line width=\plotlinewidth, smooth]   plot (\x,{0.5 * (1-(\x-0.3)*(\x-0.3) + sqrt((1-(\x-0.3)*(\x-0.3))*(1-(\x-0.3)*(\x-0.3))))});
		\draw[color=color_p1, line width=\plotlinewidth]	(1.3,0.6) node[anchor=west] {$p(T|R=1)$};
		\draw[color=color_g, line width=\plotlinewidth]    plot (\x,{0.5*(\x + sqrt(\x*\x))}) node[anchor=north west] {$f_Y(T)$};
		\draw[color=color_p0, line width=\plotlinewidth, smooth]   plot (\x,{0.5*(\x - sqrt(\x*\x))*sin(90*\x)});
		\draw[color=color_p0, line width=\plotlinewidth]	(-2, 1.2)node[anchor=south west] {$p(T|R=0)$};
		\end{tikzpicture}
	\end{minipage}
	\begin{minipage}{0.6\textwidth}
		Consider the following model with dependencies
		$R \rightarrow T \rightarrow Y$,
		think e.\,g.\ of $R=0$ as indicating samples taken in winter,
		$R=1$ samples taken in summer, $T$ is the temperature and
		$Y$ depends on temperature but only if $T>0$°C (above freezing).
	\end{minipage}
\end{example}

In this example we notice multiple unexpected properties:
\begin{enumerate}[label=(\alph*)]
	\item
		$f_Y$ depends on $T$, but this dependence becomes
		\emph{within our observations} invisible for $R=0$ ("system states" with
		$T>0$ where also $R=0$ are never reached).
	\item
		$f_Y$ itself does not actually change (it does not even depend on $R$ directly\Slash{}causally).
	\item
		Given any sort of independence model represented by a graph (\eg an LDAG \citep{LDAG_definition}),
		does it agree with absence (a) or presence (b) of the edge $T \rightarrow Y$
		for $R=0$?
\end{enumerate}

The point of view (a) prefers a "simpler" model for regime $R=0$, 
in an Occam's razor sense for \emph{this regime}, \ie 
following the philosophy
that a model for this regime should only be as complicated as it has to be
to describe this regime relative to no prior knowledge.
We will call this concept the "descriptive" graph; for the example above, it should
\emph{not} include the edge $T \rightarrow Y$. For example there would be no
reason to fit a regressor of $Y$ to $T$ in this regime.

The point of view (b) prefers a "simpler" model for regime $R=0$,
in an Occam's razor sense relative to \emph{all} the data.
It follows the philosophy that assuming causal models are robust,
we should consider validity of transfer the norm; hence we should only claim a
regime-specific model to be different from the union-model,
if there is evidence for this difference;
a model should only be as complicated as it has to be
to describe this regime relative to prior knowledge of the union-model.
We will call this concept the "transfer" (§\ref{sec:ConnectToTransfer})
or "physical" graph.
For the example above, it should
include the edge $T \rightarrow Y$.

Finally, point (c) is intimately linked to the possibility
of constraint-based discovery of these graphical objects.
We find that the SCM-centered perspective here
includes slightly different
information than many commonly used independence models (see §\ref{apdx:LDAGs}).
We will discuss the identifiability of these structures from data in detail in §\ref{sec:ConnectToCSI} and §\ref{sec:ConnectToJCI}.

Finally abstracting the underlying problem in the example \ref{example:intro} we can better understand, how it is deeply connected to generalization and transfer: From the illustration in Fig.\ \ref{fig:intro} we find the following.
For the examples on the left hand side a single mechanism is observed across multiple contexts. This means it is in principle save to learn \eg a regression on all data, transfer it between contexts, or extrapolate beyond observations of one context using the observations of the other contexts.
For the examples on the right hand side in contrast, clearly no such argument should be made.
Distinguishing the two types of changes seems thus quite relevant;
it is clearly possible in some cases: If observational support
overlaps and mechanisms change this leads to a testable discrepancy.
Yet,  at first sight, this appears to be a rather special case, and inferences beyond it seem difficult, if not impossible.

An important qualitative aspect of such changes is
the presence of dependency in a causal sense; this aspect is formally captured by
the difference between $\gDescr{G}$ and $\gPhys{G}$ (to be defined below).
This formalization of the problem allows for a systematic approach to
distinguish between these types of changes in §\ref{sec:ConnectToJCI}.
Interestingly, the properties of the surrounding causal graph can
sometimes be leveraged to draw conclusions about a particular example:
A simple example is
that changing mechanisms depend explicitly on the context $C$,
thus occur in a causal graph among the children of $C$;
if a link appears to change and
neither of its end-points could possibly be a child of $C$ (for example because
neither endpoint is adjacent to $C$), one may exclude a physical change of mechanism.
Fig.\ \ref{fig:intro} focuses on extreme cases concerning the identifiability of the type of change; intermediate cases and cases with (possibly much) more than two contexts depend strongly on assumptions about prior probabilities of different cases.

\newcommand{\scale}[0]{1.4}

\begin{figure}[ht!]
	\begin{tabular}{p{0.1\textwidth}|p{0.4\textwidth}|p{0.42\textwidth}}
        \multicolumn{3}{c}{
            \textbf{Changes in context-specific independencies (CSI)}
        }\\\hline
		 & no change of mechanism & changing mechanism:
		\textcolor{color_f}{$C=0$} vs.\ \textcolor{color_g}{$C=1$}\\
		  & $Y := f_Y(X, \eta_Y)$
        &
        {$\begin{aligned}
            Y &:= 
            (1-C) \times \bm{\textcolor{color_f}{f(X, \eta_Y)}}\\[-0.2em]
            &+ C \times \bm{\textcolor{color_g}{g(x, \eta_Y)}}
        \end{aligned}$}\\
		\hline
		$\gUnion{G}$ & $X \longrightarrow Y$ & $X \longrightarrow Y$\\
		$\gPhys{G}$ & $X \longrightarrow Y$ & $X \dashrightarrow Y$\\
		$\gDescr{G}$ & $X \dashrightarrow Y$ & $X \dashrightarrow Y$\\
		\hline
        indepen\-dencies
        &
        \multicolumn{2}{c}{
            \centering
            In \emph{all} cases:
            $X \notindependent Y$,\quad
            $X \notindependent Y | C$,\quad
            $\bm{X \independent Y | C=0}$,\quad
            $X \notindependent Y | C=1$
        }
        \\\hline
        Densities &
		\multicolumn{2}{c}{\textcolor{color_p0}{$P(X|C=0)$}, \textcolor{color_p1}{$P(X|C=1)$}}\\[0.5em]
		\hline
		overlap
		&
		\begin{tikzpicture}[domain=-2:2, scale=\scale]
			\draw[-, line width=\plotlinewidth, color=color_f] (-2,0) -- (0,0) -- (2,1);
			\draw[color=color_p0, line width=\plotlinewidth, smooth, dashed]   plot (\x,{exp(-(\x+1)*(\x+1)*3)});
			\draw[color=color_p1, line width=\plotlinewidth, smooth, dashed]   plot (\x,{exp(-(\x)*(\x)});
		\end{tikzpicture}
		&	
		\begin{tikzpicture}[domain=-2:2, scale=\scale]
			\draw[-, line width=\plotlinewidth, color=color_f] (-2,0) -- (2,0);
			\draw[-, line width=\plotlinewidth, color=color_g] (-2,0) -- (2,1);
			\draw[color=color_p0, line width=\plotlinewidth, smooth, dashed]   plot (\x,{exp(-(\x+1)*(\x+1)*3)});
			\draw[color=color_p1, line width=\plotlinewidth, smooth, dashed]   plot (\x,{exp(-(\x)*(\x)});
		\end{tikzpicture}
		\\\hline
		negligible overlap
		&
		\begin{tikzpicture}[domain=-2:2, scale=\scale]
			\draw[-, line width=\plotlinewidth, color=color_f] (-2,0) -- (0,0) -- (2,1);
			\draw[color=color_p0, line width=\plotlinewidth, smooth, dashed]   plot (\x,{exp(-(\x+1)*(\x+1)*3)});
			\draw[color=color_p1, line width=\plotlinewidth, smooth, dashed]   plot (\x,{exp(-(\x-1)*(\x-1)*3)});
		\end{tikzpicture}
		&	
		\begin{tikzpicture}[domain=-2:2, scale=\scale]
			\draw[-, line width=\plotlinewidth, color=color_f] (-2,0) -- (2,0);
			\draw[-, line width=\plotlinewidth, color=color_g] (-2,0) -- (2,1);
			\draw[color=color_p0, line width=\plotlinewidth, smooth, dashed]   plot (\x,{exp(-(\x+1)*(\x+1)*3)});
			\draw[color=color_p1, line width=\plotlinewidth, smooth, dashed]   plot (\x,{exp(-(\x-1)*(\x-1)*3)});
		\end{tikzpicture}
		\\\hline
	\end{tabular}
	\caption{In each plot, the horizontal axis displays $X$, the vertical axis displays (rescaled) probability-density for the densities (blue and green), and $Y$ for the mechanisms (black and orange).
    In the graphs at the top, dashed edges indicate edges that are only present for $C=1$ but not for $C=0$.
    See main text for further explanations.}\label{fig:intro}
\end{figure}
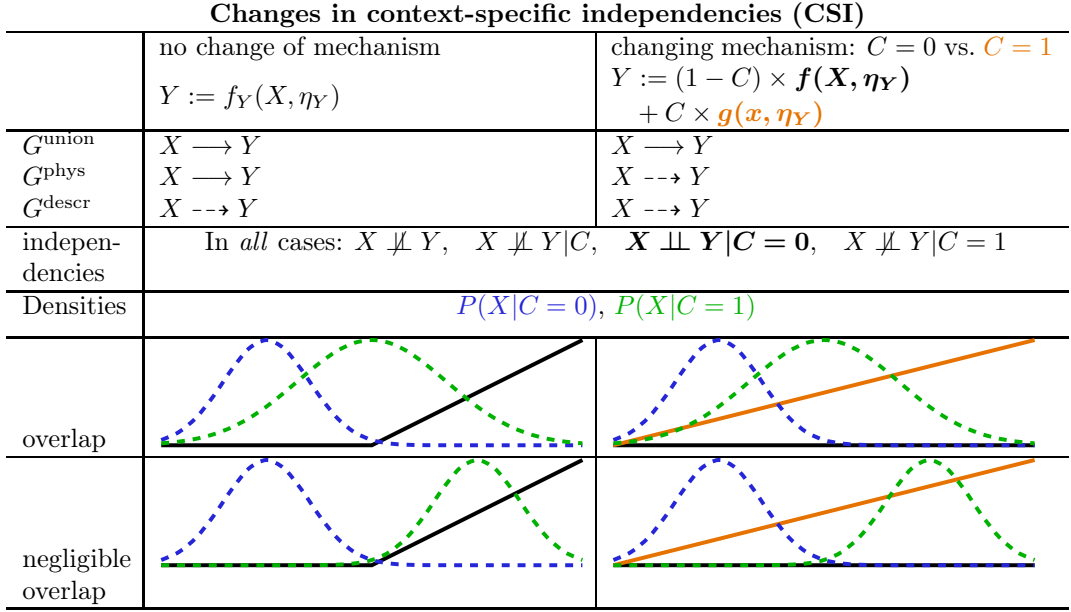

\paragraph{Contributions}
	We further study graphical objects, introduced in \citep{MethodPaper} that
	capture qualitative information about the
	causal structure plus availability of observations.
	We focus on the graphs' skeleta (that is, on their adjacencies only).
\begin{itemize}
	\item
	We extend identifiablitly results (in particular the Markov property) of \citep{MethodPaper} to the
    case where context-specific graphs do not have to agree, but may differ.
	\item
	Thus, we show that these objects are empirical, \ie can be identified as well-separated objects
    from data at least in part.
	\item
	We provide a mathematical framework based on solution-functions,
	that captures implications of CSI in terms of SCM-properties.
	We focus on a single context-indicator and skeleta, but the framework
	should allow for the derivation of more general results.
\end{itemize}

This paper is primarily concerned with the theoretical understanding of the problem at hand.
Nevertheless, much of the usefulness of the results is most easily illustrated by examples.
Therefore we collected some toy-examples in §\ref{apdx:practical_examples} to illustrate how and what conclusions can be drawn via the material introduced in the paper to solve practical problems,
a more systematic and automated approach\Slash{}method to draw such conclusions is left to future work.
In particular we continue
the example \ref{example:intro} in §\ref{apdx:practical_example_intro}.

\section{Related Literature}

A more structured overview together with additional details on some aspects like independence models can be found in §\ref{apdx:ConnectionsToLit}.

Combining datasets for causal modeling, in particular using a
context-indicator variable, has been studied extensively
to gain insights (\eg orientations) about the joint-\Slash"union"-model \citep{bareinboim2012transportability,SelectionVars, CD-NOD, JCI}.
E\,g.\ \citep{bareinboim2012transportability,SelectionVars} in particular
discuss transportability between contexts, but concerning
identifiability (structure of hidden confounders),
not available observational support.
Per-context causal models have been considered \eg by \citep{Saeed2020}, but without the descriptive vs.\ physical
distinction made here and without the connection to context-specific independence (CSI).
Graphical models for CSI in turn have been studied \eg by \citep{stratified_graphs},
or as labeled directed acyclic graphs (LDAGs) \citep{LDAG_definition, LDAG_logical},
but from the independence-model perspective, \ie without the connection to SCMs (and thus causal modeling).
Our approach opens interesting possibilities of connecting the causal and the independence-model world (§\ref{apdx:LDAGs}).
We can treat certain types of cyclic models. These are less general than those discussed in \citep{BongersCyclic}.
But we show that for these specific cyclic models (away from $R$)
the causal graph of the union-model -- not just its acyclification -- can be recovered (lemma \ref{lemma:UnionIdentifiable})
by using CSI-information.
Causal discovery with cyclic union-graphs
is also discussed \eg in \citep{hyttinen2012learning,strobl2023causal}.
The lack of observational support we study has certainly been noticed in
effect-estimation, where statements can only be made where the fit has support
-- at least for single-step adjustment \citep{Shpitser2011,Perkovic2018},
for multi-step procedures \eg the ID-Algo \citep{IDAlgo,IDAlgoMultivar} or
counter-factual queries like natural direct effects \citep{Shpitser2011}
the issue might be more subtle.
For counter-factuals more generally
similar issues have been observed \citep[§5.1]{robins2010alternative},
but no treatment from the perspective given here seems to be available.
There is also a close connection to minimality conditions and their effect
on graph-definitions.
Minimality conditions formulated in different way, for exampble:
On parent-sets in structural equations directly \citep[Def.\ 2.6]{BongersCyclic};
indirectly by requiring a minimal edge-set for which a Markov-condition holds
(dubbed "causally minimal" in \citep[§6.5.3]{Elements})
replacing a faithfulness condition (see §\ref{sec:Faithfulness});
related is also the relaxation from faithfulness to adjacency-faithfulness \citep{ramsey2012adjacency}.
Finally, a lack of observational support may be considered a missing data problem.
The literature combining missing data with causal models typically
considers latent variables \citep{spirtes2001causation,Zhang2006},
missing datasets for certain interventions (\ie interventional data is
available, but incomplete with respect to possible interventions) \citep{triantafillou2014constraint,tillman2009structure},
or robustness of causal models \citep{RojasCarulla2018, AnchorRegression}.
These topcis are quite different from the viewpoint of the problem we study.

Our novel contribution beyond \citep{MethodPaper} (see also \ref{apdx:MethodPaper} for details) is that we give a Markov property (and thus identifiability results) beyond the
case where all context-specific graphs agree.
This is relevant philosophically, to demonstrate that their distinction is indeed empirically.
It is also relevant in practise, as it allows for novel characterizations
of distribution shifts that take into account this difference
between descriptive and physical changes.
This information can be meaningful to interpret observations, see §\ref{apdx:practical_examples}.
Finally, the more general relation to independence-structure (cf.\ §\ref{apdx:LDAGs}) while capturing qualitative and relevant aspects of the observational support, seems generally interessting. Making properties of the observational support tractable is typically difficult in practice,
the qualtiative nature the information concerning the support studied here
helps to remedy what is practicly feasible with what is  interesting to understand.

\section{Causal Graphical Models}\label{sec:graph_defs}

This section contains definitions designed to formally capture the intuitions gained from example \ref{example:intro} in the introduction. These graphs have been first introduced for this
purpose in \citep{MethodPaper}, with particular
regard to endogenous context-variables.
As example \ref{example:intro} shows (see §\ref{apdx:practical_example_intro}), their usefullness
extends to exogenous context-variables as well.
It seems helpful to keep the following challenges in mind:
\begin{itemize}
	\item 
	How will we relate these definitions -- and thus the SCM and
	interventional\Slash{}causal properties
	-- to what we observe?
	Constraint-based causal discovery traditionally uses the independence-structure
	to draw conclusions about the causal structure. Including context-specific
	independence (CSI) means using observations of the un-intervened
	(\emph{not} context specific) model
	to gain context-specific information.
	We define a descriptive graph $G^{\text{descr}}_{R=r}[M]$
	for each context (or "regime") $R=r$,
	which describes \emph{from qualitative model properties}
	what CSI to expect in such observations.
	The connection to CSI will be made in §4.
	\item 
	For further inferences, we usually do not only want to know what we see,
	but also what can be concluded about the model.
	Extrapolation beyond what we observe would be difficult.
	However if we observe multiple contexts, in places where they "overlap",
	we can ask if the observations are mutually consistent (or put differently:
	If they are consistent with a shared model).
	For graphical properties, this allows for a conservative description,
	where the encoded information is only about verifiable \emph{differences}
	within a particular context relative to the full data-set.
	We define a physical graph $G^{\text{phys}}_{R=r}[M]$ to capture
	this formally. The identifiability from data
	of this object is studied in §5. It starts from the descriptive
	context-specific information by combining the $G^{\text{descr}}_{R=r}[M]$
	for different $r$.
\end{itemize}

\subsection{Structural Causal Models (SCM)}

We work within the framework of "structural causal models" (SCM) \citep{PearlBook,Elements}:
We fix a set of \defName{endogenous variables} (observable)
$\{X_i\}_{i\in I}$, for some finite $I$,
taking values in $\val{X}_i$,
and \defName{exogenous noises} (hidden) $\{\eta_i\}_{i\in I}$,
taking values in $\val{N}_i$.
We write $V := \{X_i | i\in I\}$ for the set of all endogenous variables,
$U := \{\eta_i | i\in I\}$ for the set of all exogenous noises,
and for $A \subset V$, let $\val{X}_A := \prod_{j\in A} \val{X}_j$,
further $\val{X} := \val{X}_V$ and $\val{N} := \val{N}_U$.

\begin{definition}
	A set of \defName{structural equations} (mechanisms) $\mathcal{F} := \{f_i | i\in I \}$ is an assignment
	of parent-sets $\Pa(i) \subset V$ together with mappings
	$f_i : \val{X}_{\Pa(i)} \times \val{N}_i \rightarrow \val{X}_i$ for all $i$.
	An \defName{intervention} $\mathcal{F}_{\PearlDo(A=g)}$ on a subset $A \subset V$
	is a replacement of $f_j \mapsto g_j$ for $j\in A$.
	We will only consider "hard" interventions $g_j \equiv x_j = \const$.
\end{definition}

Given a distribution $P_\eta$ of the noises $U$,
if a set of random variables $V$ solving the equations in $\mathcal{F}$
exists, we call their distribution $P_{\mathcal{F},P_\eta}(V)$
an \defName{observable distribution}. For the models we consider,
solutions are always unique and are solutions in terms of the noise-values
in the intuitive sense, §\ref{sec:solvability}.

\begin{definition}\label{def:SEM}
	An \defName{SCM} $M$ is a triple $M = ( V, U, \mathcal{F} )$,
	with $V$ distributed according to an observable distribution $P_{\mathcal{F},P(U)}(V)$.
	The \defName{intervened model} $M_{\PearlDo(A=g)}$ is an SCM with 
	$M_{\PearlDo(A=g)} = (V', U,$ $\mathcal{F}_{\PearlDo(A=g)})$
	\ie $U$ is distributed according to the same $P_\eta$ as for $M$
	and the structural equations are replaced according to the intervention.
\end{definition}

\subsection{Induced Graphical Objects}\label{sec:graphs}

An important concept in causal inference is to capture qualitative relations
between variables as described by (suitably minimal, see next sections)
parent-sets $\Pa_i \subset V$ in a \defName{causal graph},
constructed with nodes $V$ and a directed edge from each $p\in\Pa_i$ to
$X_i$.
In multi-context settings, there is
additional qualitative information available "per context", but
as explained in the introduction, multiple meaningful definitions of
parent-sets (hence graphs) exist.
The simplest way to describe qualitative properties of an SCM is
via the mechanisms only:
\begin{definition}
	Given mechanisms $\mathcal{F}$,
	the mechanism-graph $G[\mathcal{F}]$ is constructed with parent-sets $\Pa$
	such that:
	\begin{equation*}
		X \in \Pa(Y)
		\quad\Leftrightarrow\quad
		f_Y \text{ is a non-constant (in $X$) function of $X$.}
	\end{equation*}
\end{definition}
If one actually knows the underlying SCM, this is well-defined. However,
in most applications, one has limited knowledge about the "real" SCM
(approximating) a physical system and, thus, uses observed data generated by the SCM
to draw conclusions.
The choice of suitable (empirically accessible) graphical objects is intimately linked to minimality
and faithfulness assumptions (§\ref{sec:Faithfulness}, §\ref{apdx:Faithfulness}).
To capture "accessible states" we need to build
information about observational support into our
graphical objects.
\begin{definition}\label{def:obs_graph}
	Given a set of mechanisms $\mathcal{F}$,
	and a (as of now arbitrary\Slash{}unrelated to $M$ or $\mathcal{F}$)
	distribution $Q(V)$ of the observables $V$,	
	the \defName{observable graph} $G[\mathcal{F},Q]$ is constructed by defining
	parent-sets $\Pa' \subset \Pa$ such that:
	\begin{equation*}
		X \in \Pa'(Y)
		\quad\Leftrightarrow\quad
		f_Y|_{\supp(Q(\Pa(Y)))}
		\text{ is non-constant (§\ref{apdx:NonConst}) in $X$}
		\text{.}
	\end{equation*}
	Note that this depends qualitatively on $Q$, in the sense that
	it \emph{only} depends on the
    support $\supp(Q_A)$ of
	marginalizations of $Q$ to $A \subset V$.
\end{definition}

\begin{rmk}\label{rmk:replace_support_by_finitesample}
	One may replace the above definition by
	one that also includes a dependence-measure $d$
	(or rather its estimator) used to test independences,
	see §\ref{apdx:FiniteSample}.
	This seems to feature the same distinction of
	descriptive\Slash{}"detectable" vs.\ physical changes.
	But it inherently depends on finite-sample-properties; note that similar conclusions are drawn
    by \citet{janzing2023reinterpreting} from a
    completely different perspective.
	We focus on the support instead and leave finite-sample-properties to future work.
\end{rmk}

For illustratation of how the not yet fixed parameters
$\mathcal{F}$ and $Q$ can be related to a model,
we briefly excurse to causal faithfulness properties
(see also §\ref{sec:Faithfulness}).
First define an observable graph,
by fixing $\mathcal{F}$ and $Q$,
from model properties only.
\begin{definition}\label{def:visible}
	Given an SCM $M= ( V, U, \mathcal{F} )$,
	with observable variables distributed by $P_M(V)$,	
	then the \defName{visible graph}\footnote{This graph is called a union graph in \citep{MethodPaper}, and indeed it is a union graph in the sense of lemma \ref{lemma:union_is_union}.} $\gVisible{G}[M]$
	is the observable graph $G[\mathcal{F}, P_M]$.
\end{definition}
This graph moves the problem of observational support
from the faithfulness assumption
into the graph-definition (§\ref{apdx:Faithfulness}) in the following sense:
If the model $M$ is \emph{not} faithful to its visible graph $\gVisible{G}[M]$, then this failure of faithfulness must arise from a property \emph{other than observational support} (\eg from fine-tuned mechanisms,
multi-variable-synergy or deterministic relationships).
This is, in the single context case,
the same as the mechanism graph after enforcing a suitable minimality condition (like \citep[Def.\ 2.6]{BongersCyclic}) on $\mathcal{F}$.
$\gVisible{G}[M]$ is thus what would commonly be defined as "the" causal graph; in the single context
case, anything outside the support in that one context is unempirical.

The visible graph was explicitly constructed as a graph "$G[M]$" associated to an SCM.
Other than before (for $G[\mathcal{F}]$ and $G[\mathcal{F},Q]$),
there is more than one meaningful choice here!
We fix a (finite, with $P(R=r)>0$) categorical context\Slash{}"regime-indicator" variable $R$
and want to understand qualitative changes in the model between different values of $R$
(cf.\ also \citep{stratified_graphs,LDAG_definition}).
\begin{definition}\label{def:g_descr}
	Given an SCM $M= ( V, U, \mathcal{F} )$
	and $R\in V$,	
	the \defName{descriptive graph} (see \citep{MethodPaper}) is
	\begin{equation*}
		\gDescr{\bar{G}}\FixedR[M] := G[\mathcal{F}_{\PearlDo(R=r)},P_M(V|R=r)] \txt.
	\end{equation*}	
	Fixing $R$ to a value, removes dependencies involving $R$, so we add this information back in
	by defining $\gDescr{G}\FixedR[M]$ as $\gDescr{\bar{G}}\FixedR[M]$
	plus edges involving $R$ in $\gVisible{G}[M]$.
\end{definition}
This object describes the qualitative
relations between variables of the regime-"enforced" model $M_{\PearlDo(R=r)}$
that can be learned from the observed distribution $P_M$
(via conditioning) and contains the "descriptive" information
about (in)dependences we want to learn (see §\ref{sec:intro}, point (a)).

\begin{rmk}
	This graph is \emph{very} different from a "conditioned" model:
	For example there are no spurious links from selection-bias. This is, because
	this graph describes \emph{properties of the underlying SCM}
	under constraints on observable "system-states",
	and makes no reference to \eg independencies.
	It is however closely connected to independence properties
	(cf.\ §\ref{sec:ConnectToCSI}).
\end{rmk}

For edges not involving $R$, and thus for $\gDescr{\bar{G}}\FixedR[M]$,
we could replace $\mathcal{F}_{\PearlDo(R=r)}$ by $\mathcal{F}$ in definition \ref{def:g_descr},
which underlines the idea of describing an object that can be inferred
from observations, but contains information about the interventional model.
To capture §\ref{sec:intro}, point (b), we use:
\begin{definition}	
	Given an SCM $M= ( V, U, \mathcal{F} )$,
	and $R\in V$,	
	the \defName{transfer\Slash{}physical graph} is
	$\gTransf{\bar{G}}\FixedR[M]:=G[\mathcal{F}_{\PearlDo(R=r)},P_M(V)]$.
	and again $\gTransf{G}\FixedR[M]$ adds edges involving $R$.
\end{definition}
As illustrated in the introduction, this keeps links that vanish through
changing state-accessibility of the system (it keeps information available
on the pool), but deletes those that "explicitly" change via $\PearlDo(R=r)$,
\eg if $Y=\mathbbm{1}(R) \times X + \eta_Y$ (so it captures a very intuitive notion
of "per-regime" changes).
Finally interventional models -- note that Def.\,\ref{def:SEM}
keeps the exogenous noises in the definition of the intervened model, hence
it has a "counter-factual" character (§\ref{apdx:CF}) -- correspond to
\begin{definition}	Given an SCM $M= ( V, U, \mathcal{F} )$,
	and $R\in V$,	
	the \defName{counter-factual graph} is
	$\gCF{G}\FixedR[M]:=G[\mathcal{F}_{\PearlDo(R=r)},P_M(V|\PearlDo(R=r))]
	= \gVisible{G}[M_{\PearlDo(R=r)}]$.
\end{definition}
See also §\ref{apdx:CF}, where it is quickly explained why $\gCF{G}\FixedR$
seems more relevant with experimental data, we focus on observational data here.
Finally, some properties of these graphs (proofs are in
§\ref{apdx:GraphRelations}):
\begin{lemma}\label{lemma:union_is_union}	
	Union Properties, for $\gUnion{G}[M]:=\gVisible{G}[M]$:
	\begin{enumerate}[label=(\roman*)]
		\item
		$\gUnion{G}[M]$ is the "union graph" in the sense of
		\citep{Saeed2020}
		\item
		$\gUnion{G}[M]
		= \cup_r \gTransf{G}\FixedR[M]$
		\item
		$\gUnion{G}[M]
		= \cup_r \gDescr{G}\FixedR[M]$,
		if $M$ is strongly $R$-faithful (Def.\ \ref{def:strong_ff})
	\end{enumerate}	
\end{lemma}
Point (ii) is of course the motivation of (i) in \citep{Saeed2020},
but here we can explicitly see that in this case (for the union),
the specific choice of graph ($\gTransf{G}\FixedR$ vs.\ $\gDescr{G}\FixedR$) is (mostly) unimportant.
\begin{lemma}\label{lemma:edge_inclusions}
	Relations of edge-sets:
	\begin{equation*}
		\gDescr{G}\FixedR[M]
		\quad\subset\quad
		\gTransf{G}\FixedR[M]
		\quad\subset\quad
		\gUnion{G}[M]
	\end{equation*}
	writing "$G' \subset G$" if both $G$ and $G'$ are defined on the same nodes, and
	the subset-relation holds for the edge-sets.
\end{lemma}
\begin{lemma}\label{lemma:PhysChangesAreRegimeChildren}
	Physical changes are in regime-children:\\
	If $Y \neq R$ with $R\notin \gUnion\Pa(Y)$,
	then $\gPhys\Pa_{R=r}(Y) = \gUnion\Pa(Y)$.
\end{lemma}

\subsection{Potential Applications}\label{sec:applications}

\paragraph{Where are these graphs relevant?}

For applications like earth-sciences, the problem of restricted support seems to exist in practise.
Further many important applications here involve the
study of extreme events, where a restriction to small regions of the state-space
is believed to occur
\citep{MelindaExtremeValueTheoryClimate,LargeDeviationReview}
-- one possible intuition is that extremes occur from the
coincidence\Slash{}synergy of different pathways,
for example many time-steps with little precipitation
followed by high temperatures putting drought extremes in
a "corner" of the state-space.
It is often somewhat unclear why (presumably robust) causal mechanisms
seem to change under extreme conditions. Our approach provides a possible
explanation, as causal discovery (e.\,g.\ with masking, rmk.\ \ref{rmk:CD})
should typically (see §\ref{sec:ConnectToCSI}) recover $\gDescr{G}\FixedR$
at best, thus is very sensitive to observational support.
Extreme states (like droughts) are often endogenous, \ie themselves
driven by system-variables (\eg by soil-moisture feed-backs).

The setup also relates naturally to "technological" data like
satellite-telemetry or IT-safety applications, where systems behave much more like
state-machines (or actors) with many actions only available in certain states.
Note that here the state typically changes in response to sensory input, so
when modeling data about system \emph{and} environment (\eg by including data for sensory input),
the resulting contexts are typically endogenous.
While our approach is still very far from systematically recovering
a state-machine as a causal model, an understanding of the observations-support
properties studied here seems to be an important building-block when approaching this problem.
It seems noteworthy that also a causal perspective on explainable AI (XAI),
treating neural network (layers) and their inputs as SCMs, typically
have such qualitative structure, \eg from ReLU activation-functions (§\ref{apdx:practical_example_xai}).

\paragraph{What are these graphs good for?}

A common problem in practice is, given two or morecontexts
-- \eg normal data and anomalous data -- to "explain" the difference; meaningful defintions of explain become clearer in §\ref{apdx:practical_examples}.
If, between the two contexts, a mechanism $f_i$ changes its parents \emph{physically},
then this change at $f_i$ probably should be part of the explanation for changed observations.
If, on the other hand, the changes (addition, removal or combinations)
of the parents in $f_i$ do not require any explanation beyond the change in support,
\ie if they are purely descriptive (non-physical), then the explanation
for the changing observations should be found in the ancestors, not at $f_i$.
\Eg for example \ref{example:intro} in the introduction,
assuming we observe additional nodes that provide orientation-information
(or if there is a mediator $R \rightarrow M \rightarrow T$),
we note that $T \rightarrow Y$ cannot be a physical change
because $R \notin \gUnion\Pa(Y)$ (see corollary \ref{cor:R1} below).
So, instead of claiming the two contexts to differ by a change in $f_Y$
(which is indeed \emph{not} the case), we should look further upstream in the graph,
which, here, leaves $f_T$. Given $R$ the only "real" remaining degree of freedom
is $P(T|R=0) \neq P(T|R=1)$, which is a surprisingly accurate diagnosis.

Further, also interventional and counterfactual queries happen in a different
(non-union) context in terms of knowledge about mechanisms in certain value-ranges.
We consider this to be a known problem and especially for counterfactuals it has
been discussed from slightly different points of view (see §\ref{apdx:CF}).
Our treatment certainly does not suffice to "solve" this problem,
but we show that including information about knowledge
and observability into causal inference -- for multi-context data --
can (and by the motivations above, maybe should) be systematically approached.

\section{Context-Specific Independence}\label{sec:ConnectToCSI}

Note that while changes in $\gTransf\Pa(Y)$ from mechanism-changes
only occur if $R \in \gUnion\Pa(Y)$
(lemma \ref{lemma:PhysChangesAreRegimeChildren}),
"state-access" induced changes can even (also if $R$ is \emph{not} on any cycle in $\gUnion{G}$),
remove links in $\gDescr{G}$ between ancestors of $R$.
This can be undetectable even from context-specific independencies,
if the same link "should" be removed for a specific regime $r$
-- requiring us to conditioning on $R=r$ --
but gets "reinstated" by selection-bias -- because we are conditioning on $R$.
For a concrete example, see \ref{example:NonMarkov}.
This section shows that -- up to this issue (links vanishing in ground-truth
between ancestors of $R$ due to state-access restrictions) -- the descriptive
graph $\gDescr{G}\FixedR$ can be recovered from combining context-specific
and non-context-specific independence-tests, if a suitable faithfulness property 
is satisfied (§\ref{sec:Faithfulness}).
The general relation to independence structure (see point (c) in the introduction) is a bit more complicated as discussed further in §\ref{apdx:CSI}: By the counterexample \ref{example:NonMarkov} mentioned above, 
the full recovery of $\gDescr{G}$ from (context specific plus pooled) independencies alone is not possible. Further, there are limitations to our theoretical approach when studying links between ancestors of $R$ (also discussed in §\ref{apdx:CSI}).
These differences seems to be small and often can be corrected for in further analysis; for example for the identification of $\gUnion{G}$ or $\gPhys{G}$
in §\ref{sec:ConnectToJCI}, it turns out to be insubstantial.
So for most practical purposes the combined (pooled and CSI) independence-structure ($\gDetect{G}$ in §\ref{apdx:CSI}) and the theoretical guarrantees
of our Markov property ($\gIdent{G}$ below)
are very good approximations of $\gDescr{G}$.

\subsection{Markov Properties}\label{sec:Markov}

Here we study the following question:
When can the absence of an edge in the graphical objects
defined above, be detected directly from independence constraints?
Hence this section discusses Markov properties in the sense of factorization properties of the joint
probability density and the question of completeness of constraint-based causal discovery.
For DAGs these factorization properties coincide with the question "does d-separation in the
graph imply independence?", and sometimes this is taken as the definition
of a "Markov property" (\eg \citep{BongersCyclic}).

Our results considerably extend those of \citep{MethodPaper} as follows: In \citep{MethodPaper},
the primary idea to resolve the issues encountered \eg in example \ref{example:intro} by formally capturing them and giving reasonable assumptions under which they are excluded. This assumption essentially consists of enforcing equality of all context-specific graphs.
Thus giving a Markov property that applies only if all context-specific graphs agree made sense for the purposes of \citep{MethodPaper}; under such assumption the proof dramatically simplifies, as classical d-separation techniques apply at least indirectly yielding (counterfactual) independence statements. Here we are interested in how to distinguish different context-specific graphs, both to validate that these defintions are empiricial and to use those differences to solve practical problems (§\ref{apdx:practical_examples}).
As explained below, this requires new ideas and a novel framework to skip the intermediate d-separation argument typically employed to derive a global Markov-property from a local one.

\paragraph{Formalizing the Result:}
As illustrated in the introduction to this section §\ref{sec:ConnectToCSI},
we will have to exclude relations between ancestors (beyond the union-graph)
from our formal claims (see example \ref{example:NonMarkov}),
as they are not generally accessible:
\begin{definition}\label{def:anc_anc_corr}
	Define the (identifiable) ancestor--ancestor correction
	$\gIdent{G}\FixedR$ as follows:
	Start with $\gIdent{G}\FixedR = \gDescr{G}\FixedR$, then add all edges
	in $\gUnion{G}$, between any two ancestors in $\gUnion{G}$ of $R$
	to $\gIdent{G}\FixedR$. 
\end{definition}
\begin{lemma}\label{lemma:no_phys_anc_anc}
	There are no physical ancestor--ancestor problems (proof in  §\ref{apdx:GraphRelations}):\\
	$\gDescr{G}\FixedR \subset \gIdent{G}\FixedR \subset \gUnion{G}$ and
	if $M$ is strongly regime-acyclic (see below), then
	$\gIdent{G}\FixedR \subset \gPhys{G}\FixedR$.
\end{lemma}

We need some comparatively weak acyclicity assumption
called strong regime-acyclicity
\citep[Def.\ C.6]{MethodPaper},
ensuring the absense of cycles \emph{per regime}
or involving ancestors of the indicator $R$; see also §\ref{sec:solvability} for details on the definition \ref{def:acyclicity} and a connection to solvability.
Finally, the Markov-property we obtain reads:

\begin{prop}\label{prop:markov}
	Assume the model is strongly regime-acyclic
	and causally sufficient.
	If $X$ and $Y$ are non-adjacent in $\gIdent{G}\FixedR$ and both $X, Y \neq R$, then
	either
	\begin{enumerate}[label=(\alph*)]
		\item
		for $Z = \gUnion\Pa(X)$ or
		$Z = \gUnion\Pa(Y)$
		it holds $X \independent Y | Z$, or
		\item		
		for $Z = \gDescr\Pa\FixedR(X) - \{R\}$ or 
		$Z = \gDescr\Pa\FixedR(Y) - \{R\}$		
		it holds $X \independent Y | Z, R=r$.
	\end{enumerate}	
	Further, if either $X \notin \gUnion\Anc(R)$ or $Y \notin \gUnion\Anc(R)$,
	then (b) applies, otherwise (a) applies.
\end{prop}

Note the restriction on where to search for $Z$,
which is relevant for causal discovery
in practice, is a bit subtle here
(see cor.\ \ref{cor:markov}, rmk.\ \ref{rmk:CD}).
Links involving $R$ are a bit special:

\begin{rmk}\label{rmk:EdgesFromR}
	If one of the variables is $R$
	then (for univariate $R$) no regime-specific tests are available
	and we have to fall back to the "standard" result (see \eg \cite{BongersCyclic}):	
	Assume the model is causally sufficient.
	If $R$ and $Y$ are non-adjacent in $\Acycl(\gUnion{G})$,
	then there is $Z = \gUnion\Pa(R)$ or $Z = \gUnion\Pa(Y)$
	with $R \independent Y | Z$.	
	If $Y$ is an ancestor of $R$ this does not change the result
	if the model is strongly regime acyclic.
	However, if $Y$ is part of a directed cycle involving at least one child of $R$,
	then the edge $R \rightarrow Y$ in $\Acycl(\gUnion{G})$ cannot be deleted
	from our independence-constraints, even if it is not in $\gUnion{G}$.
	By the above, together with prop.\ \ref{prop:markov}, 
	this is the only such issue that can occur.
\end{rmk}

\paragraph{The Proof-Idea:}
A more detailed description and proofs can be found in §\ref{apdx:Markov}.
Here we only sketch the proof idea.
Commonly one starts from the \emph{local} Markov-property:
The idea is that
knowing $Z$ containing the parents of $Y$ renders $Y$ independent of all noises other than $\eta_Y$,
because $Y = f_Y(\Pa_y = \pa_Y, \eta_Y)$.
The subtle problem here is to (a) understand this not only for union-parents, but also for
parents in $\gDescr{G}\FixedR$ and (b) to then combine both.
The issue is that $\gDescr{G}\FixedR$ is not a causal graph associated to a SCM in the standard sense%
\footnote{In \citep{MethodPaper},
	it is shown that there are meaningful "sufficiency" assumptions,
	s.\,t.\ $\gDetect{G}\FixedR = \gPhys{G}\FixedR = \gCF{G}\FixedR = \gVisible{G}[M_{\PearlDo(R=r)}]$,
	in which case the problem (mostly) is reduced to (b).
	Here we are interested in the differences between those
	graphical objects in particular, so we need identifiability-results
	that hold more generally.}
(\ie there need not exist an SCM $M'$ with $\gVisible{G}[M'] = \gDescr{G}\FixedR$),
so the local Markov-property is not obvious anymore. We come back to this momentarily.
Seemingly this problem (a) can be resolved by considering a "conditioned" SCM (replace
$P(\eta)$ by $P(\eta|R=r)$ and keep $\mathcal{F}$ to obtain $M'$,
which confounds ancestors of $R$, see lemma \ref{lemma:split_P_eta}),
but than point (b) becomes even harder -- intuitively, since information in causal discovery
is in the missing links, one wants to combine information of link-removals from CSI
(the "conditional" graph) with link-removals from the union-model, so one is inclined to 
intersect the respective edge-sets. But problem (b) essentially asks about the connection of the resulting
object to the underlying SCM (and the regime-specific SCM $M_{\PearlDo(R=r)}$).
An important contribution of our proof is that it shows, how this information
("intersect two graphs")
is related to the SCM via $\gDescr{G}\FixedR$.
This connection allows then for further inferences in §\ref{sec:ConnectToJCI} and §\ref{sec:ConnectToTransfer}.

The way we approach the problem, is by first facing yet another subtlety:
The "propagation" of the local information from the local Markov-property
to obtain global statements about the graph is normally done via a separation-criterion,
that analyzes individual paths in the graph. But what does blocking a path in $\gDescr{G}\FixedR$
mean? The idea we propose is to go from graphical properties to conditional independencies
not via a separation-criterion (when blocking $Z$) and paths as an intermediate step,
but via changes in the noise-distribution (when conditioning on $Z$) and the form of solution-functions.
Here the "form of solution-functions" captures graphical properties,
because the system of structural equations can be solved "downstream" starting from
root-nodes, successively working down their descendants -- as follows (see §\ref{sec:solvability}):
\newtheorem*{RecallAncOnly}{Cor.\,\ref{cor:solution_depends_on_regime_anc_only}}
\begin{RecallAncOnly}
	Given a solvable, weakly regime-acyclic model,
	then, for any set of variables $X$:	
	\begin{enumerate}[label=(\alph*)]
		\item
		$F_X$ depends only on noise-terms of ancestors of $X$ in $\gUnion{G}$.
		\item
		$F^{R=r}_X := F_X|_{F_R^{-1}(\{r\})}$ depends only  on noise-terms of
		ancestors of $X$ in $\gDescr{G}\FixedR$.
	\end{enumerate}
\end{RecallAncOnly}
As the reader may have noticed we phrased the local Markov-property above via dependence of noise-terms
(rather than independence of non-descendants). Next, consider a conditioning set $Z\supset \Pa(Y)$.
The essential trick is that
since knowledge via $Z$ of the parents of $Y$ renders $Y$ independent of all noises other than $\eta_Y$,
another variable $X = F_X(\eta_A)$ is independent of $Y$ given $Z$ as long as
$\eta_A \independent \eta_Y | Z$. But again by the form of solution-functions, this time of $F_Z$,
we know which $\eta_i$ will be "mixed" (become dependent, see lemma \ref{lemma:split_P_eta})
when conditioning on $Z$.

Formulating the local Markov-property directly through "dependence on $\eta_Y$ only"
works for our setup immediately. Further it makes
problem (a) solvable since Cor.\,\ref{cor:solution_depends_on_regime_anc_only}b
is applicable for $\gDescr{G}\FixedR$!
Finally, these constraints obtained through solution-functions are uniform, in the sense
that it does not matter if we used Cor.\,\ref{cor:solution_depends_on_regime_anc_only}a or Cor.\,\ref{cor:solution_depends_on_regime_anc_only}b to obtain an intermediate result.
The obtained statements can thus be easily combined, which eventually allows to resolve problem (b).

\subsection{Faithfulness Properties}\label{sec:Faithfulness}

As is shown in \citep{MethodPaper}
(and repeated in \ref{apdx:Faithfulness})
standard faithfulness assumptions by a (short)
argument justify the following

\begin{assumption}\label{ass:faithful}
	We assume the model to be $R$-adjacency-faithful
	in the sense that for all $r$:
	\begin{equation*}
		\exists Z \txt{ s.\,t.\ }
		\left\{\quad\begin{aligned}
		X &\independent Y | Z \quad\txt{or}\\
		X, Y \neq R \txt{ and }X &\independent Y | Z, R=r
		\end{aligned}\quad\right\}
		\quad\Rightarrow\quad
		\txt{$X$ and $Y$ are not adjacent in $\gDescr{G}_{R=r}$}
	\end{equation*}
\end{assumption}

The from a theoretical point of view potentially more interesting observations is:
The CSI-Markov-property (§\ref{sec:Markov}) guarantees independences for edges
not in $\gIdent{G}\FixedR$, while the $R$-faithfulness argument only 
provides dependence-guarantees for edges in $\gDescr{G}\FixedR$.
As the counter-example \ref{example:NonMarkov} shows,
the Markov-property in general cannot hold for $\gDescr{G}\FixedR$,
but it might of course still hold for a graph $G^\txt{CSI}$ "in-between"
$\gDescr{G}\FixedR \subset G^\txt{CSI}\FixedR \subset \gIdent{G}\FixedR$.
It is unclear if such a $G^\txt{CSI}\FixedR$
for which both meaningful faithfulness- and Markov-properties hold
exists (see §\ref{apdx:Faithfulness}, §\ref{apdx:CSI}).
For the moment, we are primarily interested in relating CSI-information
to SCM-information, so we leave details on the CSI-independence-structure
of distributions induced by (\eg regime-acyclic) SCMs to future research.

Further, to recover a union-graph, we will need (see §\ref{apdx:not_strongly_faithful}
for an example, or §\ref{apdx:GraphRelations} for the proof of correctness of the union-graph
under this assumption, Lemma \ref{lemma:union_is_union} (iii)):
\begin{definition}\label{def:strong_ff}
	We say $M$ is strongly $R$-faithful, if it is $R$-faithful and the mechanisms of
	the union-model are non-deterministic, in the sense that there is no
	set of mechanisms $\mathcal{F'}$ which almost always produces the same
	observations as $\mathcal{F}$, but has different minimal parent-sets.
\end{definition}

\section{Proof of the Markov-Property}\label{apdx:Markov}

Here, the detailed proof of the Markov-property (Prop.\ \ref{prop:markov}) is presented.
See §\ref{sec:ConnectToCSI}
for a high-level overview.

We start from restrictions induced by the graphs on the form of solution-functions.
Recall from §\ref{sec:solvability} that
because the system of structural equations can be solved "downstream" starting from
root-nodes, successively working down their descendants, they depend only on noise-terms of
ancestors within the respective graph:
\newtheorem*{RecallAncOnly2}{Cor.\,\ref{cor:solution_depends_on_regime_anc_only}}
\begin{RecallAncOnly2}
	Given a solvable, weakly regime-acyclic model,
	then, for any set of variables $X$:	
	\begin{enumerate}[label=(\alph*)]
		\item
		$F_X$ depends only on noise-terms of ancestors of $X$ in $\gUnion{G}$.
		\item
		$F^{R=r}_X := F_X|_{F_R^{-1}(\{r\})}$ depends only  on noise-terms of
		ancestors of $X$ in $\gDescr{G}\FixedR$.
	\end{enumerate}
\end{RecallAncOnly2}

\subsection{Graphical Properties Reflected in the Joint Distribution} 

Next, recall that (generally) such restrictions on functional dependence translate
to product-structures on distributions as follows:

\begin{lemma}\label{lemma:split_P_indep}
	Given $A \independent B$ and a mapping $f(A)$ of $A$ only, then
	\begin{equation*}
		P(A,B| f(A)) = P(A|f(A)) \times P(B)
	\end{equation*}
\end{lemma}

\begin{proof}
	$P(A,B| f(A)) = P(B,A|f(A)) = P(B|A, f(A)) \times P(A|f(A))
	= P(B|A) \times P(A|f(A)) = P(B) \times P(A|f(A))$,
	where the last equality is by $A \independent B$
	$\Leftrightarrow$ $P(B|A) = P(B)$.
\end{proof}

We can use this, to see which part of the "noise-space" is affected by conditioning.
Note that the real power of this approach is hidden in the knowledge
about ancestral relations via Cor.\,\ref{cor:solution_depends_on_regime_anc_only}
\emph{combining} information about the two different graphs $\gUnion{G}$ and $\gDescr{G}\FixedR$.
We write "$P(\{\eta_i\})$" for $P(\eta_1, \ldots, \eta_N )$ for the $N$ noise-terms of the $N$ observables
$X_i$. We then use set-notation to make restrictions more explicit (\eg $\{\eta_i| i \in A\}$ instead of $\eta_A$).

\begin{lemma}\label{lemma:split_P_eta}
	Given a solvable, weakly regime-acyclic, causally sufficient model,
	and a set $Z$ of variables, 
	then,	
	\begin{enumerate}[label=(\alph*)]
		\item
		Using $A := \gUnion\Anc(Z)$:
		\begin{equation*}
			P(\{\eta_i\}|Z)
			= 
			P(\{\eta_i| i \in A\}|Z)
			\times
			\prod_{j \notin A}
			P(\eta_j)
		\end{equation*}
		In particular:
		\begin{equation*}
			k\notin A\quad\Rightarrow\quad
			P(\eta_k|Z)
			=
			P(\eta_k)
		\end{equation*}
		\item
		If $R \notin Z$ and fixing a value $R=r$,
		using $A_r := \gUnion\Anc(R) \cup \gDescr\Anc\FixedR(Z)$:
		\begin{equation*}
			P(\{\eta_i\}|Z, R=r)
			= 
			P(\{\eta_i| i \in A_r\}|Z, R=r)
			\times
			\prod_{j \notin A_r}
			P(\eta_j)
		\end{equation*}
		In particular:
		\begin{equation*}
			k\notin A_r\quad\Rightarrow\quad
			P(\eta_k|Z, R=r)
			=
			P(\eta_k)
		\end{equation*}
	\end{enumerate}
\end{lemma}

\begin{proof}
	(a)
	By corollary\,\ref{cor:solution_depends_on_regime_anc_only}a,	
	$F_Z$ depends only on noise-terms of ancestors $A$ of $Z$ in $\gUnion{G}$.
	In particular we can write $Z=F_Z(\{\eta_i| i \in A\})$.
	Using this and lemma \ref{lemma:split_P_indep}, which applies by causal sufficiency:
	\begin{align*}
		&P(\{\eta_i\}|Z=z)\\
		&= P\big(\{\eta_i | i\in A \}, \{\eta_j | j\notin A \}
		\quad|\quad F_Z(\{\eta_i| i \in A\}) = z\big)\\
		&= P\big(\{\eta_i | i\in A \}
		\quad|\quad F_Z(\{\eta_i| i \in A\}) = z\big)
		\times P\big(\{\eta_j | j\notin A \}\big)
	\end{align*}
	The first term is indeed just $P(\{\eta_i | i\in A \} | Z=z)$,
	while the second term is a product by causal sufficiency.
	The second claim (of part a) follows by marginalizing this.
	
	(b)	
	By corollary\,\ref{cor:solution_depends_on_regime_anc_only}a,
	$F_R$ depends only on noise-terms of ancestors of $R$ in $\gUnion{G}$.
	In particular we can write $R=F_R(\{\eta_i| i \in A_r\})$
	(with trivial dependence on elements in $A_r$ not in $\gUnion\Anc(R)$).
	
	By corollary\,\ref{cor:solution_depends_on_regime_anc_only}b,
	$F_Z^{R=r} = F_Z|_{F_R^{-1}(\{r\})}$
	depends only on noise-terms of ancestors of $Z$ in $\gDescr{G}\FixedR$.
	In particular we can write $Z=F_Z(\{\eta_i| i \in A_r\})$
	(with trivial dependence on elements in $A_r$ not in $\gDescr\Anc\FixedR(Z)$).	
	Using this and lemma \ref{lemma:split_P_indep}, which applies by causal sufficiency:
	\begin{align*}
		&P(\{\eta_i\}|Z=z, R=r)\\		
		&= P\big(\{\eta_i | i\in A_r \}, \{\eta_j | j\notin A_r \}
		\quad|\quad F_R(\{\eta_i| i \in A_r\}) = r,\quad
		F_Z(\{\eta_i| i \in A_r\}) = z \big)\\
		&= P\big(\{\eta_i | i\in A_r \}, \{\eta_j | j\notin A_r \}
		\quad|\quad F_R(\{\eta_i| i \in A_r\}) = r,\quad
		F^{R=r}_Z(\{\eta_i| i \in A_r\}) = z \big)\\		
		&= P\big(\{\eta_i | i\in A_r \}
		\quad|\quad F_R(\{\eta_i| i \in A_r\}) = r,\quad
		F^{R=r}_Z(\{\eta_i| i \in A_r\}) = z \big)
		\times P\big(\{\eta_j | j\notin A_r \}\big)
	\end{align*}
	Again, the first term is just $P(\{\eta_i | i\in A_r \} | R=r, Z=z)$,
	while the second term is a product by causal sufficiency.
	The second claim (of part b) follows by marginalizing this.
\end{proof}

\Ie on the "noise-space", selection-bias from conditioning is confined to "sources" $\eta_i$
from $A$ (or $A_r$ respectively).
The idea is now, to separate two variables, not by explicitly blocking all paths,
but by building a "barrier" $B$ to divide the system  (by conditioning)
into two regions of noise-terms affecting one
variable vs.\ those affecting the other, and using
the observation above (lemma \ref{lemma:split_P_eta}), to choose $B$ such that selection-bias
also does not mix those two regions.

Many ideas of the "standard" setup carry over, for example the "local Markov-Property"
formalizes the observation that, given its parents, a variable $X_k$,
depends \emph{only} on its "own" noise-term $\eta_k$. Hence the parents
separate the "region" containing only $\eta_k$ from all other noises (thus from upstream variables)
and if $X_k$ is not included in a directed cycle conditioning on the parents will
not induce selection-bias ($\eta_k \independent \eta_i | \Pa_k$). Here this can be formulated
as a "barrier" against all other noise-terms (lemma \ref{lemma:LocalMarkov}).

\subsection{Definitions and their Properties}

Immediately from the solution-properties Cor.\,\ref{cor:solution_depends_on_regime_anc_only},
we can relate variables to the sources of their randomness:

\begin{definition}
	Noise-sources of observations:
	\begin{enumerate}[label=(\alph*)]
		\item
		The source of a set of variables $X$ is
		$\biasSource(X) = \gUnion\Anc(X)$.
		\item
		The $r$-source is
		$\biasSource_r(X) = \gDescr\Anc\FixedR(X)$.
	\end{enumerate}	
\end{definition}

If we do not block paths, we need some other notion of separation, following
the idea of studying the changes to the noise-space:

\begin{definition}\label{def:barrier_noise}
	Separation from noise-sources:
	\begin{enumerate}[label=(\alph*)]
		\item
		A barrier $B$ separating a set of variables $Y$ from the noise-sources of
		another set of variables $C$ is
		a set of variables disjoint from $Y$ (\ie $B\cap Y = \emptyset$; but \emph{not}
		necessarily from $C$)
		such that $Y \independent \eta_C | B$.
		\item
		An $r$-barrier $B$ separating $Y$ from the noise-sources of $C$ is
		a set of variables disjoint from $Y$ with $R\in B$
		such that $Y \independent \eta_C | B', R=r$
		(where $B' = B - \{R\}$).
	\end{enumerate}	
\end{definition}

Such "barriers" exist:
The "local" Markov property,
essentially says that parent-sets (from a suitable graph),
block out all other (exogenous) noise-terms,
it can be formulated in this language as:

\begin{lemma}\label{lemma:LocalMarkov}
	Local Markov Property for Barriers (assuming causal sufficiency):
	\begin{enumerate}[label=(\alph*)]
		\item
		For any variable $Y$ which is not part of a directed cycle
		in $\gUnion{G}$,
		the set $B = \gUnion\Pa(Y)$ is
		a barrier separating $Y$ from the noise-sources of any set $C$ not containing $Y$.
		\item
		For any variable $Y$ with $R \neq Y$,
		which is not part of a directed cycle in $\gDescr{G}\FixedR$,
		and with $Y \notin \gUnion\Anc(R)$,
		the set $B = \gDescr\Pa\FixedR(Y) \cup \{R\}$ is
		an $r$-barrier separating $Y$ from the noise-sources of
		any set $C$ not containing $Y$.
	\end{enumerate}	
\end{lemma}
\begin{proof}
	(a)
	Let $B = \gUnion\Pa(Y)$, then $Y = f_Y(B=b, \eta_y)$,
	we write (for fixed $b$) $f_Y(b, -)$ for the mapping $n_Y \mapsto f_Y(b, n_Y)$
	in particular for measurable $U_Y$ and almost all $b$
	\begin{equation*}
		P( y\in U_Y | B=b ) = P( n_Y \in f_Y(b,-)^{-1}(U_Y) | B=b )\txt,
	\end{equation*}
	or written as a pushforward $P(Y|B=b) = f_Y(b,-)_*P(\eta_Y|B=b)$,
	which is determined by $P(\eta_Y|B)$.
	Since, by hypothesis, $Y$ is not part of a directed cycle
	$Y \notin \gUnion\Anc(B)$,
	thus by lemma \ref{lemma:split_P_eta}a
	(second part), $P(\eta_Y|B)=P(\eta_Y)$.
	By causal sufficiency thus $Y \independent \eta_C | B$ if $Y \notin C$.
	
	(b)
	Let $B = \gDescr\Pa\FixedR(Y) \cup \{R\}$, $B' = B - \{R\}$,
	then if $R=r$ we have almost surely $Y = f_Y(B'=b', \eta_y)$:
	By definition of $\gDescr{G}\FixedR$,
	if $R=r$ then $f_Y$ almost surely depends only on $B$ (potentially trivially on $R$)
	and $\eta_Y$.
	Thus again, for measurable $U_Y$ almost always (with $b=(b',r)$)
	\begin{equation*}
		P( y\in U_Y | B'=b', R=r ) = P( n_Y \in f_Y(b,-)^{-1}(U_Y) | B'=b', R=r )\txt,
	\end{equation*}
	or written as a pushforward $P(Y|B'=b', R=r) = f_Y(b,-)_*P(\eta_Y|B=b)$,
	which is determined by $P(\eta_Y|B)$.
	Since, by hypothesis, $Y$ is not part of a directed cycle
	$Y \notin \gDescr\Anc\FixedR(B)$.
	Further, by hypothesis, $Y \notin \gUnion\Anc(R)$,
	thus by lemma \ref{lemma:split_P_eta}b
	(second part), $P(\eta_Y|B=b)=P(\eta_Y)$.
	By causal sufficiency thus $Y \independent \eta_C | B', R=r$ if $Y \notin C$.
\end{proof}

Most importantly, "any set $C$ not containing $Y$"
in the previous lemma includes $\biasSource(X)$ if $Y \notin \Anc_{\txt{union}}(X)$
(similarly for (b)), so we will be able to relate noise-space properties back to
properties of observables.

\begin{definition}\label{def:SepOfObs}
	Separation of observables:
	\begin{enumerate}[label=(\alph*)]
		\item
		A barrier $B$ separating two sets of variables $X$
		from $Y$ is
		a barrier separating $Y$ from the noise-sources of $\biasSource(X)$,
		with $B \cap X = \emptyset$.
		(Thus $B \cap (X\cup Y) = \emptyset$, by def.\ \ref{def:barrier_noise}.)
		\item
		A $r$-barrier $B$ separating two sets of variables $X$ from $Y$ is
		a $r$-barrier separating $Y$ from the noise-sources of $\biasSource_r(X)$,
		with $B \cap X = \emptyset$.
		(Thus $B \cap (X\cup Y) = \emptyset$, by def.\ \ref{def:barrier_noise}.)
	\end{enumerate}	
\end{definition}

Note that the noise-barriers provided by the local Markov condition automatically
"qualify" to separate $X\neq R$ and $Y$ if $X$ is not a (direct) parent
(in the respective graph) of $Y$.
Further, these (def.\ \ref{def:SepOfObs})
indeed relate to independences on the observables:

\begin{lemma}\label{lemma:GeneralizedSep}
	Separation implies independence:
	\begin{enumerate}[label=(\alph*)]
		\item
		If $B$ is a barrier separating $X$ from $Y$,
		then
		$X \independent Y | B$.
		\item		
		If $B$ is a $r$-barrier separating $X$ from $Y$,
		then
		$X \independent Y | B', R=r$, with $B' = B - \{R\}$.
	\end{enumerate}	
\end{lemma}

\begin{proof}
	(a)
	By definition, a barrier $B$ between $X$ and $Y$ is
	a barrier separating $Y$ from noise of $\biasSource(X)=\Anc_{\txt{union}}(X)$.
	\Ie $Y \independent \eta_{\gUnion\Anc(X)} | B$, but by
	corollary\,\ref{cor:solution_depends_on_regime_anc_only}a,	
	$F_X$ depends only on noise-terms of ancestors of $X$ in $\gUnion{G}$,
	so that also $Y \independent F_X(\eta_{\gUnion\Anc(X)}) | B$,
	with $X = F_X(\eta_{\gUnion\Anc(X)})$ this proves claim (a).
	
	(b)
	By definition, a $r$-barrier $B$ between $X$ and $Y$ is
	an $r$-barrier separating
	$Y$ from the noise of $\biasSource_r(X)=\gDescr\Anc\FixedR(X)$.
	\Ie $Y \independent \eta_{\gDescr\Anc\FixedR(X)} | B', R=r$ (where $B' = B - \{R\}$),
	but by
	corollary\,\ref{cor:solution_depends_on_regime_anc_only}b,	
	$F_X^{R=r}$ depends only on noise-terms of ancestors of $X$ in $\gDescr{G}\FixedR$.
	By conditioning on $R=r$ we restrict ourselves to noise-terms in $F_R^{-1}(\{r\})$,
	thereby considering the restriction $F_X^{R=r} = F_X|_{F_R^{-1}(\{r\})}$
	suffices.
	Thus $Y \independent F^{R=r}_X(\eta_{\gDescr\Anc\FixedR(X)}) | B', R=r$.
	Again the claim follows by $X = F^{R=r}_X(\eta_{\gDescr\Anc\FixedR(X)})$
	(whenever defined, \ie whenever $R=r$).
\end{proof}

Now that we have a framework for replacing path-blocking arguments
in a way suitable to the problem at hand, we can return to the Markov-properties
of our systems.

\subsection{The Markov-Property}

As illustrated in the main text §\ref{sec:Markov} (see also example \ref{example:NonMarkov}),
we will have to exclude relations between ancestors (beyond the union-graph)
from our formal claims,
as they are not generally accessible (see also §\ref{apdx:CSI} however):
\theoremstyle{definition}
\newtheorem*{def_anc_anc_corr}{Definition \ref{def:anc_anc_corr}}
\begin{def_anc_anc_corr}
	Define the (identifiable) ancestor--ancestor correction
	$\gIdent{G}\FixedR$ as follows:
	Start with $\gIdent{G}\FixedR = \gDescr{G}\FixedR$, then add all edges
	in $\gUnion{G}$, between any two ancestors in $\gUnion{G}$ of $R$
	to $\gIdent{G}\FixedR$. 
\end{def_anc_anc_corr}

\begin{rmk}
	$\gUnion{G}$ and $\gTransf{G}\FixedR$ differ only by edges pointing
	into a (union-)child of $R$ (lemma \ref{lemma:PhysChangesAreRegimeChildren}),
	so "$G_{\txt{union}}$" in the definition above may
	be replaced by "$\gTransf{G}\FixedR$" as these always agree on edges between ancestors.
	In particular the "ancestor--ancestor" problem will never be an issue if
	we are interested in $\gTransf{G}\FixedR$ (see §\ref{sec:ConnectToJCI}).
\end{rmk}

Knowing, what separating barriers may look like (by the "local" Markov property
lemma \ref{lemma:LocalMarkov}), and how to use them to obtain
independence-relations on observables (def.\ \ref{def:SepOfObs},
lemma \ref{lemma:GeneralizedSep}), we finally obtain:

\theoremstyle{plain}
\newtheorem*{prop_markov}{Proposition \ref{prop:markov}}
\begin{prop_markov}
	Assume the model is strongly regime-acyclic
	and causally sufficient.
	If $X$ and $Y$ are non-adjacent in $\gIdent{G}\FixedR$ and
	both $X, Y \neq R$, then
	either
	\begin{enumerate}[label=(\alph*)]
		\item
		for $Z = \gUnion\Pa(X)$ or
		$Z = \gUnion\Pa(Y)$
		it holds $X \independent Y | Z$, or
		\item		
		for $Z = \gDescr\Pa\FixedR(X) - \{R\}$ or 
		$Z = \gDescr\Pa\FixedR(Y) - \{R\}$		
		it holds $X \independent Y | Z, R=r$.
	\end{enumerate}	
	Further, if either $X \notin \gUnion\Anc(R)$ or $Y \notin \gUnion\Anc(R)$,
	then (b) applies, otherwise (a) applies.
\end{prop_markov}
\begin{proof}
	Case 1 (both $X$ and $Y$ are (union-)ancestors of $R$):
	By strong regime-acyclically \asswlog $Y \notin \gUnion\Anc(X)$.
	In this case, by construction of $\gIdent{G}\FixedR$, $X$ and $Y$ are (non-)adjacent
	in $\gIdent{G}\FixedR$ if and only if they are (non-)adjacent in $\gUnion{G}$,
	thus $X \notin \gUnion\Pa(Y)$.
	By the local Markov-property lemma \ref{lemma:LocalMarkov}a -- which applies,
	because $Y$ is not part of any union-cycle by strong regime-acyclicity --
	$Z=\Pa_{\txt{union}}(Y)$ is a barrier separating $Y$ from the noise of
	$\gUnion\Anc(X)$. As noted above $X \notin \gUnion\Pa(Y) = Z$, so
	this is a barrier separating $X$ from $Y$.
	Therefore, by lemma \ref{lemma:GeneralizedSep}a, $X \independent Y | Z$ as claimed.
	
	Case 2 (\asswlog $Y \notin \gUnion\Anc(R)$):
	Note that we can further assume \asswlog $Y \notin \gDescr\Anc\FixedR(X)$,
	because, if we had $Y \in \gDescr\Anc\FixedR(X)$:
	\begin{enumerate}
		\item[Case 2a]
		($X \in \gUnion\Anc(R)$):
		Then if it were
		$Y \in \gDescr\Anc\FixedR(X) \subset \gUnion\Anc(X)$ (by lemma \ref{lemma:edge_inclusions}),	
		this would imply $Y \in \gUnion\Anc(R)$ in contradiction to the hypothesis of the case 2.
		\item[Case 2b] ($X \notin \gUnion\Anc(R)$), then by weak regime-acyclicity,
		$X \notin \gDescr\Anc\FixedR(Y)$ and we can swap	the roles of $X$ and $Y$ to satisfy
		the \asswlog assumption of the case
		and $Y \notin \gDescr\Anc\FixedR(X)$.
	\end{enumerate}	
	Thus by lemma \ref{lemma:LocalMarkov}b -- which applies, because
	by weak regime-acyclicity $Y$ is not part of any cycle in $\gDescr{G}\FixedR$
	and $Y \notin \gUnion\Anc(R)$ by hypothesis of the case --
	using $Z = \gDescr\Pa\FixedR(Y)$, we find
	$Z \cup \{R\}$ is a $r$-barrier separating $Y$ from the noise of $X$.
	Again $X \notin \gUnion\Pa(Y)$ and $X \neq R$, so $X \notin Z \cup \{R\}$,
	and this is a barrier separating $X$ from $Y$.
	By lemma \ref{lemma:GeneralizedSep}a, $X \independent Y | Z, R=r$ as claimed.
\end{proof}

\theoremstyle{definition}
\newtheorem*{rmk_edges_from_r}{Remark \ref{rmk:EdgesFromR}}
\begin{rmk_edges_from_r}
	If one of the variables is $R$
	then (for univariate $R$) no regime-specific tests are available
	and we have to fall back to the "standard" result (see \eg \cite{BongersCyclic}):	
	Assume the model is causally sufficient.
	If $R$ and $Y$ are non-adjacent in $\Acycl(\gUnion{G})$,
	then there is $Z = \gUnion\Pa(R)$ or $Z = \gUnion\Pa(Y)$
	with $R \independent Y | Z$.	
	If $Y$ is an ancestor of $R$ this does not change the result
	if the model is strongly regime acyclic.
	However, if $Y$ is part of a directed cycle involving at least one child of $R$,
	then the edge $R \rightarrow Y$ in $\Acycl(\gUnion{G})$ cannot be deleted
	from our independence-constraints, even if it is not in $\gUnion{G}$.
	By the above, together with prop.\ \ref{prop:markov}, 
	this is the only such issue that can occur.
\end{rmk_edges_from_r}

The restriction on where to search for $Z$ is relevant for causal discovery algorithms
in practice, and the following reformulation is helpful to that end:

\begin{cor}\label{cor:markov}
	Given a strongly regime-acyclic, causally sufficient model,
	$X$, $Y$ not adjacent in $\gIdent{G}\FixedR$ and both $X, Y \neq R$,
	then either
	\begin{enumerate}[label=(\alph*)]
		\item
		it exists $Z \subset \gIdent\Adj\FixedR(X)$ or
		$Z \subset \gIdent\Adj\FixedR(Y)$
		with $X \independent Y | Z$, or
		\item		
		it exists $Z \subset \gIdent\Adj\FixedR(X) - \{R\}$ or 
		$Z \subset \gIdent\Adj\FixedR(Y) - \{R\}$		
		with $X \independent Y | Z, R=r$.
	\end{enumerate}
\end{cor}
\begin{proof}
	We have to show that the conditioning sets in proposition\,\ref{prop:markov}
	are in the adjacencies of $\gIdent{G}\FixedR$.	
	If (b) applies, then either
	$Z \subset \gDescr\Pa\FixedR(X) \subset \gIdent\Pa\FixedR(X) \subset \gIdent\Adj\FixedR(X)$
	or $Z \subset \gDescr\Pa\FixedR(Y) \subset \gIdent\Pa\FixedR(Y) \subset \gIdent\Adj\FixedR(Y)$
	and there is nothing to show.
	If either $X \notin \gUnion\Anc(R)$ or $Y \notin \gUnion\Anc(R)$,
	then (b) applies.
	So the only remaining case is where both $X$ and $Y$ are in $\gUnion\Anc(R)$.
	In this case, since parents of ancestors of $R$ are again ancestors of $R$,
	and edges between nodes in $\gUnion\Anc(R)$ are in $\gUnion{G}$
	if and only if they are in $\gIdent{G}\FixedR$ we have (from part (a))
	$Z \subset \gUnion\Pa(X) = \gIdent\Pa\FixedR(X)\subset \gIdent\Adj\FixedR(X)$
	or $Z \subset \gUnion\Pa(Y) = \gIdent\Pa\FixedR(Y) \subset \gIdent\Adj\FixedR(Y)$.
\end{proof}
\begin{rmk}\label{rmk:CD}
	There is still a subtle difficulty left:
	Generally, there is no reason why a model -- even if it is faithful to $\gDescr{G}\FixedR$ --
	would be faithful to $\gIdent{G}\FixedR$.
	We cannot guarantee links as in example \ref{example:NonMarkov}
	will be deleted, but they \emph{might} be nevertheless (see also §\ref{apdx:CSI}).
	So generally by causal discovery using the proposition,
	one finds a graph $\gDetect{G}\FixedR$, with
	$\gDescr{G}\FixedR \subset \gDetect{G}\FixedR \subset \gIdent{G}\FixedR$,
	but for rule (a) one has to test all conditioning-sets contained
	in the parents in $\gIdent{G}\FixedR$.
	
	An "easy" fix would be to first discover the (acyclification of) the union graph
	with standard methods, and restrict the search for separating-sets 
	by $\gIdent{G}\FixedR \subset \gUnion{G} \subset \Acycl(\gUnion{G})$
	to $\Acycl(\gUnion{G})$. This will do more tests than actually required however.
	
	In practice it might be preferable to either:
	\begin{enumerate}[label=(\roman*)]
		\item
		Learn $\Acycl(\gUnion{G})$, then
		$\gMask{G}\FixedR$ by masking on $R=r$
		(only using rule (b), avoiding the problem discussed above)
		and then consider "intersection graphs"
		$(\gDetect{G})'\FixedR := \Acycl(\gUnion{G}) \cap \gMask{G}\FixedR$,
		which in the end also deletes all edges that can be deleted either by
		(a) or by (b).
		\item
		Find suitable assumptions that allow for more efficient
		(requiring fewer test, on the pooled data when consistent)
		algorithms \citep{MethodPaper}.
	\end{enumerate}
	While the first option sounds simpler and theoretically elegant,
	the issue of state-access induced vanishing of links between
	ancestors of $R$ precluding required tests (by searching the
	wrong adjacencies) in the indicated way
	seems a bit esoteric for most potential applications,
	with stability on finite data being a major concern for causal discovery,
	the second option certainly mandates closer investigation.
\end{rmk}

\subsection{Detectable Graph}\label{apdx:CSI}

As briefly discussed in §\ref{sec:Faithfulness},
there is a gap between links that always can be removed $\gIdent{G}\FixedR$ (prop.\ \ref{prop:markov})
and those that never will be removed $\gDescr{G}\FixedR$ (by faithfulness, ass.\ \ref{ass:faithful}).
In part, this gap is genuine -- counterexamples exist (example \ref{example:NonMarkov})
-- but in cases where $X$ and $Y$ are ancestors of $R$, but not both direct parents of $R$,
there might be more generally applicable result than prop.\ \ref{prop:markov}:
From a path-separation perspective, a path containing $R$ as collider being opened by conditioning on $R$
could still be blocked "elsewhere" along  the path.

We do not know, if single graph encoding all independence constraints while also being consistent
for conclusions drawn via paths exists.
However, there is a practical approach via directly encoding
independence structure (slightly different from LDAGs, see §\ref{apdx:LDAGs})
with the connection to SCMs given by prop.\ \ref{prop:markov} and assumption \ref{ass:faithful}
encoded in lemma \ref{lemma:g_detect}:
\begin{definition}\label{def:g_detect}
	Define the "detectable" (independence-)graph $\gDetect{G}\FixedR$ as
	the "causally minimal" representation \citep[§6.5.3]{Elements}:
	There is an edge between $X$ and $Y$ if there is no $Z$ with
	$X, Y \notin Z$ for which at least one of the independence $X \independent Y | Z$ or
	(if $X,Y \neq R$) the CSI $X \independent Y | Z, R=r$ holds.
	Orient edges not involving $R$ as in $\gUnion{G}$ (this is well-defined, by lemma \ref{lemma:g_detect} and
	lemma \ref{lemma:no_phys_anc_anc} showing $\gDetect{\bar{G}}\FixedR \subset \gUnion{\bar{G}}$)
	and edges out of $R$ not in $\gUnion{G}$, see rmk.\ \ref{rmk:EdgesFromR}, are oriented out of $R$,
	all other edges involving $R$ are also oriented as in $\gUnion{G}$.
\end{definition}
\begin{lemma}\label{lemma:g_detect}:
	Connection of $\gDetect{G}\FixedR$ to SCM:\\
	\begin{equation*}
		\gDescr{\bar{G}}\FixedR \subset \gDetect{\bar{G}}\FixedR \subset \gIdent{\bar{G}}\FixedR
	\end{equation*}
	For edges involving $R$, $\gDetect{G}\FixedR$ contains at least the edges in $\gUnion{G}$,
	but  may additionally contain edges in $\Acycl(\gUnion{G})$ out of $R$.
\end{lemma}
\begin{proof}
	The first inclusion is by ass. \ref{ass:faithful},
	the second one by prop.\ \ref{prop:markov}.
	The last statement follows from rmk. \ref{rmk:EdgesFromR}.
	By strong regime-acyclicity, there are no additionally edges in $\Acycl(\gUnion{G})$ into $R$.
\end{proof}
This provides a tight enough connection between CSI-structure and SCMs for
the arguments in §\ref{sec:ConnectToJCI}.
In practice the results in §\ref{sec:ConnectToJCI}
work for 
\begin{equation*}
	\gDescr{\bar{G}}\FixedR \subset \gDetect{\bar{G}}\FixedR \subset \gPhys{\bar{G}}\FixedR
\end{equation*}
which has the advantage of physical changes being restricted
to regime-children (lemma \ref{lemma:PhysChangesAreRegimeChildren}),
which reduces the search-space for CSI-testing and
allows for more efficient methods \cite{MethodPaper}.

\subsection{Counter-Example to General Case}

The following example illustrates the problem of links between ancestors,
vanishing by observational access, becoming invisible due to selection bias.
See start of §\ref{sec:ConnectToCSI}.

\begin{example}\label{example:NonMarkov}
	"Selection-bias between ancestors can lead to violations of the Markov-property":
	
	Let $X, Y \in U := \{ a_0, a_1, b_0, b_1 \}$ categorical variables.
	Let $X = \eta_X$ (with $P(\eta_X) > 0$),
	fix $A := \{ a_0, a_1 \} \subset U$, and $B := \{ b_0, b_1 \} \subset U$
	and the "letter" $l$ and "index" $i$ indicators on $U$ as follows
	\begin{align*}
		l: U \rightarrow \{a,b\},
		&\begin{cases}
			a_0, a_1 \mapsto a \\
			b_0, b_1 \mapsto b
		\end{cases}\\
		i: U \rightarrow \{0,1\},
		&\begin{cases}
			a_0, b_0 \mapsto 0 \\
			a_1, b_1 \mapsto 1
		\end{cases}
	\end{align*}
	Then, define
	for $\eta_Y \in \{ 0, 1 \}$ (with $P(\eta_Y) > 0$):
	\begin{equation*}
		Y = f_Y(X,\eta_Y) :=
		\begin{cases}
			a_{i(X) \xor \eta_Y} \text{ if } X \in A \\
			b_{\eta_Y} \text{ if } X \in B
		\end{cases}
	\end{equation*}	
	(On binary variables, the natural choice of
	binary operators are those of boolean algebra, \ie of the field $\mathbb{Z}/2\mathbb{Z}$,
	so that "$\xor = +$" and "$\boolAnd = *$". If the reader feels confused by the $\xor$ notation,
	they may think "$+$" (formally mod 2) instead.)
	
	Note that $l(X) = l(Y)$, and $Y$ clearly depends on $X$ in general.
	However, for $X \in B$, $Y$ does \emph{not} further depend on the value within $B$ taken by $X$, \ie $f_Y|_B$ is independent of $X$.
	
	Finally, the "regime-indicator" $R \in U$
	for $\eta_R \in \{ 0, 1 \}$ (with $P(\eta_R) > 0$ and
	$P(\eta_R=1) = p \neq \onehalf$):
	\begin{equation*}
		R =		l(X)_{\eta_R \xor i(X) \xor i(Y)}
	\end{equation*}
	
	This construction has the following interesting properties:
	$l(R) = l(X)$, hence $l(R) = b \Leftrightarrow l(X) = b$,
	therefore $\supp P(X|R=b_0) = B$.
	But $f_Y|_B$ is independent of $X$ (see above), so $X \not\in\Pa_{G_{R={b_0}}}(Y)$.
	
	However, due to selection bias, this non-adjacency is never detectable:
	Given $R=b_0$, we know $l(X) = l(R) = b$. Thus also $l(Y) = b$.
	Further, knowing $0 = i(R) = \eta_R \xor i(X) \xor i(Y)$,
	we can use information about $X$ to infer the following.
	If $i(X) = 0$, then the equation above becomes
	$0 = i(R) = \eta_R \xor i(Y) \xor 0$ with $P(\eta_R=1) = p$,
	thus $P(i(Y)=1 | R=b_0, X=b_0) = 1 - p$.
	On the other hand if $i(X) = 1$, then the equation above becomes
	$0 = i(R) = \eta_R \xor i(Y) \xor 1$ with $P(\eta_R=1) = p$,
	thus $P(i(Y)=1 | R=b_0, X=b_1) = p$.
	
	If it were $X \independent Y | R=b_0$,
	then $P(i(Y) | R=b_0, X) = P(i(Y) | R=b_0)$ would hold,
	thus also $p = P(i(Y)=1 | R=b_0, X=b_1) = P(i(Y)=1 | R=b_0)
	= P(i(Y)=1 | R=b_0, X=b_0) = 1-p$.
	But we assumed $p \neq \onehalf$.´
	So $X \notindependent Y | R=b_0$ must hold,
	and we will \emph{always}
	fail to delete this link from conditional independences alone.
\end{example}

\section{Joint Causal Inference and Transfer}\label{sec:ConnectToJCI}

The previous section explained, how (most of) the information of $\gDescr{G}\FixedR[M]$ can
be recovered from (testable) independence-constraints (prop.\ \ref{prop:markov} and ass.\ \ref{ass:faithful}),
leading to a graph (see §\ref{apdx:CSI}) $\gDetect{G}\FixedR$
with $\gDescr{G}\FixedR \subset \gDetect{G}\FixedR \subset \gIdent{G}\FixedR$.
Here we study $\gTransf{G}\FixedR[M]$ and $\gUnion{G}[M]$.
We do not know,
if $\gTransf{G}\FixedR[M]$ is fully identifiable in general,
or if the set of rules we provide is complete.
It demonstrates however that these graphs contain empirically testable information (see also example \ref{example:p1}
and discussion thereafter).
We refer to these rules (§\ref{sub:JCIlike}) as "JCI-like" as they resemble
\citep{SelectionVars,JCI}. (Proofs are in §\ref{apdx:ConnJCI}.)

\subsection{Inferring the Union-Graph}\label{sec:union_graph}

Recall from remark \ref{rmk:EdgesFromR},
that edges from $R$ into directed union-cycles containing a child of $R$ cannot be
deleted by our independences. We will hence mostly focus on edges elsewhere in the graph
("away from $R$"),
using the "barred" notation ($\gDescr{\bar{G}}\FixedR$ etc.).
Generally, a causal model is only Markov to the acyclification (see \eg \citep{BongersCyclic})
of its visible ("standard") graph $\Acycl(\gVisible{G}[M])$
while, for strongly regime-acyclic models we here have:

\begin{lemma}\label{lemma:UnionIdentifiable}
	Let $M$ be a strongly $R$-regime-acyclic, strongly
	$R$-faithful, causally sufficient model, then
	\begin{equation*}
		\gVisible{\bar{G}}[M] = \gUnion{\bar{G}}[M]
		= \cup_r \gDetect{\bar{G}}\FixedR[M]
	\end{equation*}
	is identifiable away from $R$ by ($R$-context-specific) independences.
\end{lemma}
For edges out of $R$ no context-specific tests are available,
so (see Rmk.\ \ref{rmk:EdgesFromR}):
$\gVisible{G}[M] = \gUnion{G}[M] \subset \gUnionId{G}[M] := \cup_r \gDetect{G}\FixedR[M]$,
where the difference $\gUnionId{G}[M] - \gUnion{G}[M]$ consists of	edges from $R$ to nodes in union-cycles only.

\subsection{Interring the Physical Graph by	JCI-like Rules}\label{sub:JCIlike}

Similarly, there are properties of $\gPhys{G}\FixedR$
that can be identified from data. We already know
$\gDetect{\bar{G}}\FixedR[M]\subset\gTransf{\bar{G}}\FixedR[M]\subset\gUnion{\bar{G}}[M]$
by lemma \ref{lemma:no_phys_anc_anc},
where the left-hand-side is, by construction §\ref{apdx:CSI},
identifiable (under our assumptions)
from data via prop.\ \ref{prop:markov} and lemma \ref{ass:faithful},
and the right-hand-side is identifiable by lemma \ref{lemma:UnionIdentifiable} above.
So it will suffice, for understanding $\gTransf{\bar{G}}\FixedR[M]$,
to study edges in $\gUnionId{\bar{G}}[M] - \gIdent{\bar{G}}\FixedR[M]$
and decide if those should be in $\gTransf{\bar{G}}\FixedR[M]$ or not.
As already noted in lemma \ref{lemma:PhysChangesAreRegimeChildren},
physical changes occur only in regime-children:

\begin{lemma}\label{lemma:R1}
	If $R \notin \gUnion{\Anc}(Y)$, then
	$\gTransf\Pa_{R=r}(Y) = \gUnion\Pa(Y)$,
	\ie the change is non-physical (by observational non-accessibility).
\end{lemma}
\begin{cor}\label{cor:R1}
	If $R \notin \gUnionId{\Anc}(Y)$, then
	$\gTransf\Pa_{R=r}(Y) = \gUnionId\Pa(Y)$.
\end{cor}

If (conditioning on) $R$ does not change the distribution of
ancestors, no state-induced effects occur:

\begin{lemma}\label{lemma:R2}
	Assuming strong regime-acyclicity.
	If $X \in \gUnion\Pa( Y ) - \gIdent\Pa\FixedR( Y )$
	and $R \in \gUnion\Pa( Y )$,
	and	$\gUnion\Anc(R) \cap \gUnion\Anc(\gUnion\Pa(Y)-\{R\}) = \emptyset$,
	then
	$X \notin \gTransf\Pa( Y )$ (\ie the change is "physical" not
	just by state).
\end{lemma}

\begin{cor}\label{cor:R2}
	Assuming strong regime-acyclicity.
	If $R\neq X \in \gUnionId\Pa( Y ) - \gIdent\Pa\FixedR( Y )$
	and $R \in \gUnionId\Pa( Y )$,
	and	$\gUnionId\Anc(R) \cap \gUnionId\Anc(\gUnionId\Pa(Y)-\{R\}) = \emptyset$,
	then
	\begin{enumerate}[label=(\alph*)]
		\item
		there is a link into the strongly connected component of $Y$
		that vanishes in $\gTransf{G}$, but not in $\gUnionId{G}$,
		\ie there is a physical change.
		\item
		if $Y$ is not part of a directed union-cycle, then
		$X \notin \gTransf\Pa( Y )$,
		\ie there is a physical change of this particular link.
	\end{enumerate}
\end{cor}

\subsection{Validity of Transfer}\label{sec:ConnectToTransfer}

JCI-arguments (§\ref{sec:ConnectToJCI}) can exclude the possibility of physical changes,
but they can only provide direct evidence in rare cases (lemma \ref{lemma:R2}).
But variable can depend quantitatively on $R$:
\begin{example}\label{example:qualitative_R_dependence}
	If $Y = g(X) + \gamma R + \eta_Y$, with $\gamma \neq 0$, and
	if $g|_{\supp P(X|R=r_0)}$ is constant, $X \in \gPhys\Pa_{R=r_0}(Y)$,
	even though  $X \notin \gDescr\Pa_{R=r_0}(Y)$
	and $R \in \gUnion\Pa(Y)$.	
\end{example}
We sketch a statistical test (see also §\ref{apdx:TransferForPhys}),
that approaches this problem in analogy to the philosophy of
constraint-based causal discovery (CD): For CD, the idea is that in an Occam's razor sense,
a link should be considered relevant to the causal model,
if there is evidence for the link to be present, \ie if independence can be rejected
(see discussion of point (a) after example \ref{example:intro}).
For the multi-context case, from the perspective that causal mechanisms are supposed to be robust,
a reasonable null-hypothesis is, to assume that $g$ (in example \ref{example:qualitative_R_dependence}) remains unchanged in the context $R=r_0$.
So a link should be removed relative to the union-model if there is evidence for its vanishing
(see discussion of point (b) after example \ref{example:intro}).

In the example above, $g$ is identifiable (in $\gUnion{G}$),
so we can learn $g$ from data. Now, if we can show that the independence-test
we used for CD (of $\gDescr{G}_{R=r_0}$, see Rmk.\ \ref{rmk:CD}),
would have (likely) rejected the independence $X \independent Y |R=r_0$
given the observational distributions (\eg bootstrapping from the observational distributions)
if $g$ had remained valid in $R=r_0$,
then we have evidence for $g$ vanishing in $R=r_0$.
This formally is captured by the difference of $\gUnion{G}$ and $\gPhys{G}_{R=r_0}$
in the sense of Rmk.\ \ref{rmk:replace_support_by_finitesample}.

\section{Practical Examples / Potential Use-Cases}\label{apdx:practical_examples}

Here we illustrate the potential usefulness of our ideas. To this end,
the ideas from the main text are applied to an (extended version) of the example \ref{example:intro}
in the introduction and a slightly more complex example afterwards.
Finally we give a brief application to explainable AI questions,
which is less practical, but seems interesting to include, as it
demonstrates that these questions arise in very different settings.

\subsection{Example from the Introduction}\label{apdx:practical_example_intro}

We start with example \ref{example:intro}. 
Further, recall from the introduction what was labeled as point (c) there:
Under purely descriptive changes as in example \ref{example:intro},
links into nodes which are not children of $R$ can change; for physical
changes this cannot happen.

\begin{example}[Example \ref{example:intro} continued]
	The original example in the introduction was as follows:\\
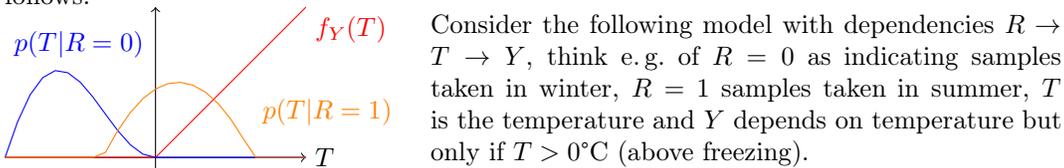
	\begin{minipage}{0.4\textwidth}
		\begin{tikzpicture}[domain=-2:2]
			\draw[->] (-2,0) -- (2,0) node[right] {$T$};
			\draw[->] (0,-0.2) -- (0,2);
			
			\draw[color=red]    plot (\x,{0.5*(\x + sqrt(\x*\x))}) node[anchor=north west] {$f_Y(T)$};
			\draw[color=blue]   plot (\x,{0.5*(\x - sqrt(\x*\x))*sin(90*\x)});
			\draw[color=blue]	(-2, 1.2)node[anchor=south west] {$p(T|R=0)$};
			\draw[color=orange]   plot (\x,{0.5 * (1-(\x-0.3)*(\x-0.3) + sqrt((1-(\x-0.3)*(\x-0.3))*(1-(\x-0.3)*(\x-0.3))))});
			\draw[color=orange]	(1.3,0.6) node[anchor=west] {$p(T|R=1)$};
		\end{tikzpicture}
	\end{minipage}
	\begin{minipage}{0.6\textwidth}
		Consider the following model with dependencies
		$R \rightarrow T \rightarrow Y$,
		think e.\,g.\ of $R=0$ as indicating samples taken in winter,
		$R=1$ samples taken in summer, $T$ is the temperature and
		$Y$ depends on temperature but only if $T>0$°C (above freezing).
	\end{minipage}
	
	We modify this as follows:
	Assume we measure a mediator $M$ between seasonal information $R$ and
	temperature $T$ so that the ground-truth causal (union) model
	looks like $R \rightarrow M \rightarrow T \rightarrow Y$.
	
	The independence-structure one finds corresponds to the following
	(descriptive) graphs (this graph is "partial",
	no orientations are found without further information):
	If $R=1$ the descriptive graph is $R - M - T - Y$,
	if $R=0$ the descriptive graph is $R - M - T \quad Y$, lacking
	a link between $T$ and $Y$.
	By Lemma \ref{lemma:UnionIdentifiable} the (partial) union graph is $R - M - T - Y$.
	Thus, by Lemma \ref{lemma:no_phys_anc_anc}
	the (partial) physical graph for $R=1$ is also $R - M - T - Y$.
	Interestingly by Corollary \ref{cor:R1} (and using that non-adjacent implies non-parent)
	the (partial) physical graph for $R=0$ is again $R - M - T - Y$.
	
	\emph{Practical conclusion:}
	While for $R=0$ the variable $Y$ becomes independent of $T$, we have no
	evidence for the underlying physical mechanism between $T$ and $Y$ changing with $R$.
	Further, the vanishing of this link can "witness"
	but not be responsible for the "differentness" of the data in regime $R=0$.
	The vanishing of the link is a consequence not a "cause" of the altered behavior.
	There must be an additional change in ancestors of $T$ driving the behavior.
	
	\emph{Note:} The results of §\ref{sec:ConnectToCSI} and §\ref{sec:ConnectToJCI} have lead us to this conclusion from
	independence properties of the data, without knowledge the model.
\end{example}

This argument requires only standard techniques (independence-testing,
see also Rmk.\ \ref{rmk:CD}), so it can be readily made actionable via standard implementations.

\subsection{More Complex Example}\label{apdx:practical_example_abstract}

The example above ends up being simple, because the conclusions can be drawn from
graphical structure already. This also means that the physical change is ruled out by
model-properties ("globally") and cannot depend on the context. Given that we do not have
to know the model, only the data, to draw this conclusion makes it still quite non-trivial.
A more complicated example
can provide further insight:

\begin{example}
	A solar-powered weather-station, one day, detects anomalous power supply. Comparing the data of the last hours to normal data, we notice that the commanded orientation of the solar-panels normally -- but no longer -- affects available power. We go out and fix the pointing mechanism. Another time, same problem, same result from data, but the anomalous data is from during the night.

	\emph{Questions:}
	Why do we draw a different conclusion? How can the distinction be formally captured and automatically detected?
	What does it mean to 'cause' a distribution shift?
	
	We approach this as follows:	
	Treat error-states as context, e.\,g.\ $E=0$ all normal data,
	$E=1$ anomaly 1, $E=2$ anomaly 2.
	Assume there are "many" instances (data-sets or day-night cycles)
	providing normal data.
	For both anomalies the descriptive graph changes in the same way
	(the power-generation becomes independent of commanded panel-orientation).
	But by the transfer-ideas outlined in §\ref{sec:ConnectToTransfer} one can
	learn that during the one anomaly happening at night
	the conditional distribution of
	light-intensity is trivially "singular" near $0$, so
	we could not have seen this link either way.
	Thus, it is \emph{not} removed in the physical graph
	(there is no evidence for doing so).
	For the anomaly happening at daylight, by the same argument,
	the link would be visible if present,
	so there \emph{is} evidence for its absence, and it is
	removed in the physical graph.
	
	\emph{Practical conclusion:}
	For the daylight anomaly, the physical mechanism must have changed,
	so assuming there was no co-occurrence of multiple changes to the system
	we should focus on the physical alignment-mechanism.
	For the nighttime anomaly, there is no evidence for this physical
	change. Further, to explain the "differentness" from the normal
	data, some other change must have happened.
	
	\emph{Answers to the questions:}
	\begin{itemize}
		\item
		Why do we draw a different conclusion?
		-- We draw a different conclusion, because once we have evidence for
		a particular change in the physical system (for the daylight anomaly,
		alignment-mechanism must have physically changed),
		while in the other case, there is no evidence for this specific kind of problem
		(which does of course not exclude it, but other potential issues are just as likely).
		\item 
		How can the distinction be formally captured and automatically detected?
		-- The distinction is formally captured by the particular change
		(for the daylight anomaly) corresponding to a link-removal in the physical graph
		relative to the union-graph. For the nighttime anomaly, no such physical change
		is detectable from data.
		This analysis can be automated, whenever the physical graph can be identified from data.
		Some criteria on when this is possible can be found in §\ref{sec:ConnectToTransfer}.
		\item 		 
		What does it mean to 'cause' a distribution shift?
		-- If the underlying physical model changed in a way that is relevant to
		the data-generation, then this change is part of the reason why the context
		appears differently; if we further exclude the occurrence of multiple changes at once,
		it is the (singular) cause of the distribution shift.
		If the description changed (the link in the descriptive graph vanishes),
		but no physical change of the mechanism is required to explain the observations
		(the link in the physical graph does \emph{not} vanish), than this particular change
		did \emph{not cause} the distribution shift; it still narrows down the location of the root
		cause to ancestors of the parent-node (in the vanishing link).
		This is to some degree a reverse formulation of JCI (where the context causes
		the variable-change), but applied to links (it is not exactly 1:1 to JCI-results,
		cf.\ §\ref{sec:ConnectToTransfer})
	\end{itemize}
\end{example}

\subsection{Connections to Explainable AI}\label{apdx:practical_example_xai}

The following is an abstract but quite interesting observation.
We start from the general idea that there seems to be some interest in a more
causal understanding of deep learning.
A (dense) neural network can be seen as an SCM: The nodes of one layer
are functions of (all) the nodes in the previous layer. So all the
nodes of the previous layer with non-zero (or sufficiently far from zero) weights are parents.
Typically on (dense) networks this leads to a -- not particularly insightful -- graph,
with an edge from each node in one layer to each node in the next layer.
However, restricting inputs to some context
(e.\,g.\ cat pictures\footnote{Let's ignore for now that one would probably use a convolutional network then.}),
only some "neurons" are active, and detections can be "explained" (in an XAI sense)
from parts of the input (in fact one often wants to understand \emph{which} parts).

This seems contradictive: If the graph is fully connected (between consecutive layers),
why is it also sparse (for explanation)?
From the present perspective, the descriptive graph (under the distribution of cat pictures)
is sparse! Many links vanish because, for this subset of possible inputs, the activation-functions
are (essentially) constant.

This is not necessary helping in doing XAI yet, but we think it provides an interesting
new perspective on the apparent contradiction encountered when trying to apply causal
logic to NNs. Also, the fact that this example is very far from the initial intuition
that lead us to define our graphical objects (mostly applications from earth-system science)
may be seen as some indication that the objects we defined do capture a rather general and
relevant abstract property.

Finally, if the context is the output of an anomaly-detection algorithm,
the interpretations above (§\ref{apdx:practical_example_intro}, §\ref{apdx:practical_example_abstract})
may be seen as an explanation of that output. The output might be triggered by a descriptive change
so this provides information beyond "heatmapping-like" analysis.

\section{Conclusion}\label{sec:conclusion}

The assumption of positivity, $P>0$, is quite common and very useful.
However, it is not popular for its realism
-- finite data never gives empirical evidence outside a bounded support, even more
so in light of Rmk.\ \ref{rmk:replace_support_by_finitesample} --
but because it dramatically simplifies the problem, by neglecting
"purely formal" details that supposedly would not actually affect the conclusions
we draw.
Generally, this is certainly often true, but as we point out, there are
a range of difficulties, where our \emph{qualitative} understanding
relies on the the understanding of available observational support.
We formally capture such qualitative properties through our
descriptive and physical graphs -- this is illustrated in the example in §\ref{apdx:practical_example_abstract},
where once a physical and once a descriptive change occurs.
Further, as we demonstrate, in multi-context systems, these qualitative
properties become accessible, at least in part, from observations.
Finally, we hope that the deeper understanding of the connection
between the structure of context-specific independencies and SCMs
via the graph objects of \citep{MethodPaper} provide
may help to better connect both worlds.

\paragraph{Future-Work:}
We focused on iid-data here, but time-series data seems like an interesting,
even though potentially quite complicated, extension.
For time-series, for example with persistent (slowly changing) regimes,
the observable support of the stationary distribution should play an important
and interesting role. What sounds very technical, captures some intuitive effects:
As an example, consider a crossroads, where in one context
(state of traffic-lights) the traffic flows
in one direction, in the other context in the orthogonal direction. Now if
states normally only accessible (by the stationary distribution)
in one context (traffic in direction A) at a regime-boundary "lag" into the other
context (traffic in direction B), then new phenomena arise.
Stationary distributions are an example of distributions that arise
from a model in a quite non-trivial way; generally observational
support in real-world applications may often be of such sophisticated nature.
In principle observable graphs $G[\mathcal{F},Q]$ can account for these cases,
yet there is certainly much left to be understood.

A more immediate generalization would be in (transfer of) orientations.
One can of course use standard orientation-rules per-context, or JCI-rules
on the union, but really one would want to combine information from both
where possible.
Clearly one would also want to understand effect estimation problems from
the perspective of these graph objects. Especially optimality results
taking into account transferability properties seem like a challenging
but relevant topic.

\paragraph{Limitations:}
Our approach captures only qualitative (graphical) changes. Our faithfulness treatment seem somewhat ad hoc.
There are fundamental limitations to independence based discovery of the
discussed graph objects (see counterexample \ref{example:NonMarkov}).
The Markov-property we give may still not be reaching these
fundamental limits (\ie it might be possible to give a slightly stronger
formulation, cf.\ §\ref{apdx:CSI}).
We currently do not consider models with hidden confounders,
we formally only allow for a single context indicator,
and we assume IID
data; all three of these are unlikely to be fundamental and it should be possible to tackle them in future work.
We do not analyize finite sample properties, error-propagation or quality measures for results like \citep{faller2024self,janzing2023reinterpreting}.

\section*{Acknowledgments and Disclosure of Funding}

The authors want to thank the other members of the causal inference lab, in particular Oana Popescu and Urmi Ninad for many helpful and inspiring discussions.

J.\,R.\ has received funding from the European Research Council (ERC) Starting Grant CausalEarth under the European Union’s Horizon 2020 research and innovation program (Grant Agreement No. 948112).
J.\,R.\ and M.\,R.\ have received funding from the European Union’s Horizon 2020 research and innovation programme under grant agreement No 101003469 (XAIDA).
W.\,G.\ was supported by the Helmholtz AI project CausalFlood (grant no. ZT-I-PF-5-11). 
This work was supported by the German Federal Ministry of Education and Research (BMBF, SCADS22B) and the Saxon State Ministry for Science, Culture and Tourism (SMWK) by funding the competence center for Big Data and AI "ScaDS.AI Dresden/Leipzig".

\bibliography{BibTex.bib}

\appendix

\tableofcontents

\section{Details on Related Literature}\label{apdx:ConnectionsToLit}

The topic presented here has connections to many fields, so we give a more structured overview below.
Also the connection to CSI and independence models \citep{stratified_graphs,LDAG_definition,LDAG_logical}
seems interesting (§\ref{apdx:CSI}), but since we expect most potential readers to come from the causal community,
a detailed treatment in the main-text seems ill-placed. Similarly, further details on the connection to
counter-factuals are in §\ref{apdx:CF}.

\subsection{Structured Overview}

\paragraph{Combining Datasets} from different contexts
in causal inference has been studied
\eg by \citep{bareinboim2012transportability, SelectionVars, CD-NOD, JCI}.
The focus there is usually on gaining orientation-information or
statistical power on finite data, \ie gaining additional information about
the union-model. The main technical ingredient
is in adding the context-information (\eg an index) to the pooled dataset
as a "context-variable" and to then study the resulting system.
We adopt this convention and call this (categorical) variable $R$ ("regime").
\citep{bareinboim2012transportability,SelectionVars} in particular
discuss transportability between contexts, but concerning
identifiability (structure of hidden confounders),
not available observational support.
For example \citep{Saeed2020} also explicitly studies graphical models for
mixtures, we will for example connect our results to their union-graph.
Their focus is in defining graphs for the combined dataset,
we focus on different graphs for a \emph{single} context.
The reason why there is not a unique (empirically accessible)
graph is that we "enrich" this single context
by our understanding of the other contexts.
So our study also is inherently multi-context, but does not focus on
the union-model (see \ref{sec:union_graph} however).

\paragraph{Context Specific Independence (CSI) and Graphical Models}
have been combined from the perspective of encoding information about (the factorization of)
a probability distribution, \eg as stratified graphs \citep{stratified_graphs},
or labeled directed acyclic graphs (LDAGs) \citep{LDAG_definition, LDAG_logical}.
The main distinction here is that we are interested primarily in \emph{causal}
properties and to this end study connections to SCMs,
thereby \eg to interventional properties.
We also establish how one of the graphical models we define relates to CSI.
The information our objects encode is subtly different from that encoded in LDAGs,
or their induced "context-specific LDAGs" \citep[Def.\ 8]{LDAG_definition}.
Both encode CSI properties however, and it should be possible via our results
to leverage results about LDAGs for causal inference, and vice versa (for example
the construction of counter-examples like \ref{example:NonMarkov} seems
much more accessible from an SCM perspective). See also §\ref{apdx:LDAGs}.

\paragraph{Cyclic Models and Solution Functions}
have gained increased attention recently.
There are other approaches to study cyclic union-models
\eg \citep{hyttinen2012learning,strobl2023causal} -- cyclic union-models are a possible use-case for
context-specific graphs, but not the core content of this paper.
The type of cyclicity we allow in our models is extremely simple compared to the general
treatment \citep{BongersCyclic}, even though they are not simple in the sense defined there.
Simple SCMs \citep{BongersCyclic} are cyclic models defined such that their solution-properties
(and simple-ness) is stable under interventions and marginalization.
It seems to be the case that the ensuing problems (in particular unsolvable intervened models),
occur, if we intervene to a system-state outside of the observational support,
so in our "support-aware" philosophy, we should capture the problem "beforehand":
We recognize the intervention as involving a transfer-problem and are thus
warned that it may not have a unique or clear solution without further information.
We do not study this connection in detail here, but we use
solution-functions in §\ref{sec:Markov}.

\paragraph{Interventions and Counterfactuals}
For interventions, \eg by single-step adjustment
\citep{Shpitser2011,Perkovic2018}, a lack of support often
becomes evident by a lack of training-data, and is comparatively easy to detect
and simple to deal with (require expert-knowledge for extrapolation, there is not much
to be done from data alone).
For multi-step procedures (like the ID-Algo \citep{IDAlgo,IDAlgoMultivar})
and especially counterfactual quantities
(like natural direct effects \citep{Shpitser2011}) the situation becomes much more complicated.
Here the question about which graphical properties even \emph{could} be learned
from data have been discussed,
see \eg \citep[§5.1]{robins2010alternative},
even though the
systematic connection to observational support does not seem
to have been studied yet.
See also §\ref{apdx:CF}.

\paragraph{Minimality and Faithfulness}
are also strongly intertwined with how to pick "the correct" graphical models.
The most direct approach is by a minimality-condition
on parent-sets \citep[Def.\ 2.6]{BongersCyclic}
(even though there is a faithfulness assumption about non-determinism
implicit for minimal parent-sets to be well-define\Slash{}unique).
For skeleton-discovery, we are primarily interested in
adjacency-faithfulness \citep{ramsey2012adjacency},
but \eg \citep[§6.5.3]{Elements}
also formalize a "causally minimal" condition which is 
faithfulness in the sense of
independences only occur where they are guaranteed
by the Markov-property, which turns out to
be quite non-trivial here (§\ref{sec:Faithfulness}).
The context-specific absence of edges itself can be
understood as a violation of faithfulness to the union-graph
(as noted \eg by \cite[§4.3.7]{JCI}.

\paragraph{Missing Data} in causal modeling in the literature
usually concerns either latent variables
\citep{spirtes2001causation,Zhang2006},
or more abstractly missing data for certain interventions
\citep{triantafillou2014constraint,tillman2009structure}
typically for the combination of datasets (see above)
or robustness of causal models \citep{RojasCarulla2018, AnchorRegression}.
The lack of overlap of observations and non-constant
mechanism domain seems so far unexplored --
certainly people are and have been aware of this issue,
but the formal and systematic approach given here
seem to be new (see also §\ref{sec:applications}).

\subsection{Relation to Method-Paper}\label{apdx:MethodPaper}

The problem of an ambiguity in the definition of per-context graphs and its connection
to observational access was encountered in \citep{MethodPaper} during the development of
a constraint-based causal discovery method for this setting.
There, the focus is on giving meaningful assumptions, under which this problem does 
not occur (\ie assumptions under which $\gDetect{G}\FixedR = \gPhys{G}\FixedR = \gCF{G}\FixedR$),
and when efficient (using few tests) causal discovery is possible.
For the scenario with $\gDetect{G}\FixedR = \gPhys{G}\FixedR = \gCF{G}\FixedR$ a
Markov-property is shown with (modified) standard tools (path-blocking),
plus some tricks involving counter-factuals (see also the footnote in §\ref{sec:Markov}).
The main distinction here is that we focus on the usefulness, and identification from data,
of the \emph{difference between physical and descriptive changes}.
This also means that a Markov-property that holds only under the assumption
of $\gDetect{G}\FixedR = \gPhys{G}\FixedR = \gCF{G}\FixedR$ is insufficient.
The more general case shown here, requires a completely different approach (§\ref{sec:Markov}, §\ref{apdx:Markov}).
The subsequent study of union and \emph{physical} graph, is relative to
a suitable proxy $\gDetect{G}\FixedR$ of $\gDescr{G}\FixedR$ (see §\ref{apdx:CSI}),
for which, in light of prop.\ \ref{sec:Markov} as presented here (see also \citep[Rmk.\ 4.2 on Thm.\ 1]{MethodPaper}),
efficient causal discovery algorithms as in \citep{MethodPaper} are suitable.
Generally, \citep{MethodPaper} evolves around developing assumptions for a method (and a method),
that is both efficient and interpretable in terms of SCMs despite the difficulties that arise from these observations.
The present paper is about studying the emerging structure: How do the different per-context
graphs relate to each other and to the union-graph, which intuition do they capture, and how can they
-- in particular also the physical graph and their differences --
be identified?

\subsection{Connection to Independence Structures}\label{apdx:LDAGs}

We briefly recall the concept of labeled acyclic directed graphs LDAGs \citep{LDAG_definition}.
The underlying system is considered to consist of categorical variables only.
Traditionally, the graphical representation of the independence-structure
represents dependencies with links, independencies with missing links,
in a sparse sense, \ie if $X \independent Y | Z$ the link from $X$ to $Y$ is also removed.
The LDAG then labels these edges with a "stratum" \citep{stratified_graphs} by the following
idea (for simplicity we pretend we knew orientations): If $X \rightarrow Y$ then
test for each combination of values of (other)
parents $Z = \Pa(Y) - \{X\}$ of $Y$ if $X \independent Y | Z=z$, in this case add $Z=z$ as a label to the edge.
In practice, some PC-like search-procedure can be used \citep{LDAG_learning}.

This, in our language, essentially treats every variable as a regime-indicator,
thus also contains the information of any specific choice of regime-indicator
(called "context-specific LDAG" in \citep[Def.\ 8]{LDAG_definition}).
The full LDAG thus contains more information than only that of a context-specific LDAG
corresponding to one choice of regime-indicator.
The price for this additional generality is the restriction of the setup to categorical variables
only, and for discovery from finite data, in cases where one is interested in a specific
context-specific LDAG, the detour through the full LDAG is likely not sample-efficient.
We think, it is also not to be underestimated that LDAGs are hard to read, compared to
the simpler (because less information-dense) context-specific ones.

Generally, the information encoded in a context-specific LDAG is very similar to our graphical models,
there are some things to note however: The context in LDAGs is local -- only strata of parents (adjacencies) 
are encoded -- while our graphs also capture non-local effects (\eg insert a mediator $R \rightarrow M \rightarrow T$
in example \ref{example:intro}, then $T \rightarrow Y$ vanishes for specific $R$-contexts, even though
$R$ is not adjacent to either $T$ or $Y$), which is also accessible from observations \eg through intersection-graphs
(Rmk.\ \ref{rmk:CD}).
We do not know if the authors of \citep{stratified_graphs, LDAG_definition, LDAG_logical, LDAG_learning},
were aware of this specific problem, \eg the formulation used by
\citet[Conjecture 1 (p.\,985)]{LDAG_logical}
about the completeness of their CSI-separation rules
relative to a hypothesis (as complete for information contained in)
$I_{\text{loc}}$, which contains this "local" CSI-information only,
\ie they were clearly very careful in avoiding potentially
false claims or conjectures about such non-local problems.
They also use positivity of the distribution,
via the semi-graphoid axioms (see also last line on \citep[p.\,983]{LDAG_logical}).

As \citet[p.\,983]{LDAG_logical} (and others) point out, generally conclusions about information contained in
a given CSI-structure (\ie which other independencies can be
derived) is a very hard problem (cf.\ \citep[§2.3]{LDAG_logical}).
The results on this topic that \citep{stratified_graphs, LDAG_definition, LDAG_logical, LDAG_learning} and
the "Bayesian network\Slash{}independence-structure community" in
general provide
could be interesting to the causal community
(\eg for cross-validation of causal discovery results),
and vice versa \eg the construction of counter-examples
like example \ref{example:NonMarkov} that are "easy" in the SCM formalism might provide insights for the
"Bayesian network community". Further the specific type of
non-local CSI we encounter seems potentially interesting for understanding independence-structures as well. We hope our approach opens new connections between both perspectives.

In \citep{LDAG_logical}, also a connection to support-properties is used
to connect results to the abstract framework of "databases and team semantics".
There the abstract model seems to describe the following observation: Given independence
the joint distribution is a product, thus its support has certain symmetry-properties.
What we study here is the overlap of observational supports and the support of
mechanisms, thus a completely different concept.

\subsection{Connection to Counter-Factuals}\label{apdx:CF}

If we are worried about selection-bias, the systematic machinery
developed for such questions is the do-calculus.
While the "mutilated" graph $G_{\bar{R}}$ defined graphically
(does not see qualitative change of mechanisms such as for
$Y = \mathbbm{1}(R) \times X + \eta_Y$)
we may ask: What is the "correct" graph for the intervened model?

This requires additional information about the exogenous noises we consider.
The most consistent approach seems to be assuming that the exogenous noises
are \emph{not} affected by the intervention in the model. In this case
this becomes the counter-factual model \citep{PearlBook} describing the world that would have been
observed (given the "circumstances" encoded in exogenous noises) if $R$ had been intervened
to be $r$. We will hence call this concept the "counter-factual" graph.
This is mostly a matter of perspective and to avoid overloading the
term intervened graph typically used in the sense of mutilated graph
(see above)
with the do-calculus.
For the example above, the counter-factual graphs seems to be the "descriptive" graph,
but this is a coincidence, and is generally only true if $R$ is exogenous.
Indeed the counter-factual graph can even have \emph{more} edges than the union graph:

\begin{example}\label{example:CFgraph}
	The Counterfactual Graph can have more Edges than the Union:\\
	\begin{minipage}{0.4\textwidth}
		\begin{tikzpicture}[domain=-2:2]
			\draw[->] (-2,0) -- (2,0) node[right] {$x$};
			\draw[->] (0,-0.2) -- (0,2);
			
			\draw[color=red]    plot (\x,{0.5*(\x + sqrt(\x*\x))}) node[right] {$g(X)$};
			\draw[color=blue]   plot (\x,{0.5*(\x - sqrt(\x*\x))*sin(90*\x)});
			\draw[color=blue]	(-2, 1.2)node[anchor=south west] {$p(R=1)$};
		\end{tikzpicture}
	\end{minipage}
	\begin{minipage}{0.6\textwidth}
		Consider the following model with descriptive graph
		\begin{equation*}
			X \rightarrow R \rightarrow Y \txt,
		\end{equation*}
		where $f_Y(X,R) = \mathbbm{1}(R) \times g(X) + R + \eta_Y$.
		Values of $X > 0$ and $R=1$ together are \emph{never} observed,
		so $Y$ seems to be independent of $X$.
	\end{minipage}
\end{example}

Here, the shown graph is the "correct" one by the usual means, but note,
that intervening on $R$ can make the link $X\rightarrow Y$ visible!
In fact, if we have interventional ("experimental") data, than this is potentially testable
in a multi-context setup, and should be considered a meaningful object.
See \citep[§5.1]{robins2010alternative} however, where related problems (in a single-context setting)
are discussed, and a number of subtleties are pointed out.
We will subsequently focus on purely observational data and leave the problem of experimental
data to future work.

Finally, we note that this counter-factual model -- under suitable assumptions --
can be used as a mathematical trick to proof a (weaker version of) the Markov-property
through "standard" path-blocking arguments \citep{MethodPaper} (because $\gCF{G}\FixedR[M] = \gVisible{G}[M_{\PearlDo(R=r)}]$
is a causal graph associated to a causal model in the standard sense).

\section{Properties of Graphs}\label{apdx:GraphRelations}

\subsection{Proofs of Statements in the Main Text}

In §\ref{sec:graphs} we gave some properties of the
studied graphical objects, here we give the corresponding proofs.
We start -- in slightly altered order compared to the main text -- with
\begin{lemma}[Lemma \ref{lemma:edge_inclusions}]
	\label{lemma:edge_inclusions:apdx}
	Relations of edge-sets:
	\begin{equation*}
		\gDescr{G}\FixedR[M]
		\quad\subset\quad
		\gTransf{G}\FixedR[M]
		\quad\subset\quad
		\gUnion{G}[M]
	\end{equation*}
	writing "$G' \subset G$" if both $G$ and $G'$ are defined on the same nodes, and
	the subset-relation holds for the edge-sets.
	Generally (\ie it can happen that) $\gCF{G}\FixedR[M] \not\subset \gUnion{G}[M]$
	(see example \ref{example:CFgraph}) and
	$\gDescr{G}[M] \not\subset \gCF{G}\FixedR[M]$.
\end{lemma}
\begin{proof}
	This follows directly from the definitions, by
	$\supp(P(\Pa(X)|R=r)) \subset P(P(\Pa(X)))$.
\end{proof}

\begin{lemma}[Lemma \ref{lemma:PhysChangesAreRegimeChildren}]
	\label{lemma:PhysChangesAreRegimeChildren:apdx}
	Physical changes are in regime-children:\\
	If $X \in \gPhys\Pa_{R=r}(Y)$, and $Y \neq R$ with $R\notin \gUnion\Pa(Y)$,
	then $\gPhys\Pa_{R=r}(Y) = \gUnion\Pa(Y)$.
\end{lemma}

\begin{proof}
	By definition, $\gUnion{G}[M] = \gVisible{G}[M] = G[\mathcal{F}, P_M]$
	and $\gPhys{G}_{R=r}[M] = G[\mathcal{F}_{\PearlDo(R=r)}, P_M]$.
	By definition, $\mathcal{F}$ and $\mathcal{F}_{\PearlDo(R=r)}$
	differ only in $f_R$ and (by setting the parameter $R=r$)
	for $f_i$ with $R \in \gUnion\Pa(X_i)$.
	For $Y$, by hypothesis neither of these two applies, so
	the same $f_Y$ is in $\mathcal{F}$ and $\mathcal{F}_{\PearlDo(R=r)}$.
	Since both graph-definitions further use the same support (that of $P_M$),
	their parent-definitions for $Y$ agree: $\gPhys\Pa_{R=r}(Y) = \gUnion\Pa(Y)$.
\end{proof}

\begin{lemma}[Lemma \ref{lemma:union_is_union}]	
	Union Properties, for $\gUnion{G}[M]:=\gVisible{G}[M]$:
	\begin{enumerate}[label=(\roman*)]
		\item
		$\gUnion{G}[M]$ is the "union graph" in the sense of
		\citep{Saeed2020}
		\item
		$\gUnion{G}[M]
		= \cup_r \gTransf{G}\FixedR[M]$		
		\item
		$\gUnion{G}[M]
		= \cup_r \gDescr{G}\FixedR[M]$,
		if $M$ is strongly $R$-faithful (Def.\ \ref{def:strong_ff})
	\end{enumerate}
\end{lemma}
\begin{proof}
	(i)
	$\gVisible{G}[M]$ corresponds to the causal graph in the
	standard sense given a suitable minimality-condition on 
	parent-sets (see §\ref{sec:graphs}),
	so it is the graph of the union-model in the sense of \citep{Saeed2020}.
	
	(ii)
	"$\supset$": By lemma \ref{lemma:edge_inclusions:apdx} $\gTransf{G}\FixedR[M] \subset \gUnion{G}[M]$,
	so $\cup_r \gTransf{G}\FixedR[M] \subset \gUnion{G}[M]$.
	
	"$\subset$": 
	Let $X \in \gUnion\Pa(Y)$ be arbitrary.
	By lemma \ref{lemma:PhysChangesAreRegimeChildren:apdx}
	we only have to consider links to regime-children $Y$.
	By definition $X \in \gUnion\Pa(Y)$ means,
	there are values $\pa, \pa' \in \supp(P(\gUnion\Pa(Y)))$
	which differ only in their $X$-coordinate
	(\ie $\pa = (x, \pa^-)$, $\pa' = (x', \pa^-)$ with
	$\pa^-$ the same value for $\gUnion\Pa(Y)-\{X\}$)
	such that $f_Y(\pa) \neq f_Y(\pa')$.
	Since $R \in \gUnion\Pa(Y)$, the tuple $\pa^-$
	also contains a value $r_1$ for $R$.
	For this $r_1$ we have $X \in \gPhys\Pa_{R=r_1}(Y)$,
	because
	$\gPhys{G}_{R=r_1} = G[\mathcal{F}_{\PearlDo(R=r_1)},P(V)]$
	uses the same support (that of $P(V)$) as
	$\gUnion{G} = \gVisible{G} = G[\mathcal{F},P(V)]$
	and $f'_Y\in\mathcal{F}_{\PearlDo(R=r_1)}$
	(forcing $R=r_1$ in $f_Y$)
	does agrees with the original $f_Y\in\mathcal{F}$ for
	$\pa$ and $\pa'$ (as they contain $R=r_1$),
	so $f'_Y(\pa) = f_Y(\pa) = f_Y(\pa') = f'_Y(\pa')$.
	
	(iii)
	"$\supset$": By lemma \ref{lemma:edge_inclusions:apdx}, as in (ii).
	
	"$\subset$":
	Let $X \in \gUnion\Pa(Y)$ be arbitrary.
	Define $N_r := F_R^{-1}(\{r\})$ (where $F_R$ is the solution-function
	for $R$ in terms of noises, see §\ref{sec:solvability}), and note that,
	by $F_R$ being a well-define mapping, $P(\vec{\eta} \in \cup_r N_r) = 1$,
	using $V_r := F_{\gUnion\Pa(Y)}(N_r)$
	and $F_{\gUnion\Pa(Y)}(\cup_r N_r) = \cup_r V_r$ thus
	so $P(\pa \in \cup_r V_r)=1$.
	
	By contradiction: Assume it were $X \notin \gDescr\Pa\FixedR(Y)$ for all $r$.
	Then, by definition, $f_Y|_{V_r}$ is constant in $X$ with probability $1$.
	We can thus define $g_Y^r(\gUnion\Pa(Y)-\{X\})$ such that
	$P(f_Y = g_Y^r | R=r) = 1$.
	Finally construct $f'_Y(\gUnion\Pa(Y)-\{X\}, R) := g_Y^R(\gUnion\Pa(Y)-\{X\})$
	(\ie depending on the value $r$ of $R$ choose the corresponding $g^r$).
	Then for $\mathcal{F'}$ defined as $\mathcal{F}$ with $f_Y$ replaced by $f'_Y$,
	the same observations are obtained with probability $1$, but parent-sets
	differ for $Y$.
\end{proof}

During the discussion of the Markov-property (§\ref{sec:Markov}) the graph $\gIdent{G}\FixedR$ is introduced,
and the following property is claimed:

\theoremstyle{plain}
\newtheorem*{lemma_no_phys_anc_anc}{Lemma \ref{lemma:no_phys_anc_anc}}
\begin{lemma_no_phys_anc_anc}\label{lemma:no_phys_anc_anc:apdx}
	There are no physical ancestor--ancestor problems:\\
	$\gDescr{G}\FixedR \subset \gIdent{G}\FixedR \subset \gUnion{G}$ and
	if $M$ is strongly regime-acyclic, then
	$\gIdent{G}\FixedR \subset \gPhys{G}\FixedR$.
\end{lemma_no_phys_anc_anc}
\begin{proof}
	Using $\gDescr{G}\FixedR \subset \gUnion{G}$ (lemma \ref{lemma:edge_inclusions}),
	by definition \ref{def:anc_anc_corr}, $\gIdent{G}\FixedR \subset \gUnion{G}$.
	The first inclusions is also by definition.
	
	"$\gIdent{G}\FixedR \subset \gPhys{G}\FixedR$":
	Let $e = (X,Y)$ and edge in $\gIdent{G}\FixedR$.\\
	Case 1 ($X, Y \in \gUnion\Anc(R)$):
	By $\gIdent{G}\FixedR \subset \gUnion{G}$ (see above),
	$e \in \gUnion{G}$.
	By lemma \ref{lemma:PhysChangesAreRegimeChildren},
	$\gUnion{G}$ and $\gTransf{G}\FixedR$ differ only by edges pointing
	into a (union-)child of $R$.
	By strong regime-acyclicity,
	children of $R$ are not union-ancestors of $R$,
	so $e \in \gTransf{G}\FixedR$.\\
	Case 2 (otherwise):
	By definition $e \in \gDescr{G}\FixedR$ in this case.
	So by lemma \ref{lemma:edge_inclusions}, $e \in \gTransf{G}\FixedR$.
\end{proof}

\subsection{Formalization of Non-Constant on Support}\label{apdx:NonConst}

In Def.\ \ref{def:obs_graph}, we require the restriction of $f_Y$ to the support
of a distribution $Q(\Pa(Y))$ to be non-constant in $X$.
Usually this can be thought of as: $\exists \pa^-$, values of
$\Pa(Y)-\{X\}$, and $x$, $x'$ values of $X$ such that
$(\pa^-,x), (\pa^-,x') \in \supp(Q(\Pa(Y)))$ and
$P(f_Y(\pa^-,x,\eta_Y) \neq f_Y(\pa^-,x',\eta_Y)) > 0$.
Formally this requires regularity-assumptions (\eg there are
continuous densities, and the $f_i$ are continuous) to 
exclude degenerate cases like:
\begin{example}
	Let $Q(X)$ uniform over $(\mathbb{R}-\mathbb{Q})\cap[0,1]$,
	and $f_Y(X,\eta_Y) = \mathbbm{1}(X\in\mathbb{Q}) \times X + \eta_Y$.
	Then $\supp(Q(X)) = [0,1]$ (it is defined as the closure, which includes
	the rationals), and $f_Y$ is non-constant on $[0,1]$,
	but really $f_Y$ would never "see" the dependence on $X$.
\end{example}
The more relevant extension to our setting seems to be the finite-sample
case §\ref{apdx:FiniteSample}.
Nevertheless, the above problem could be fixed, \eg by defining
"non-constant on the support" as:
$\exists U, U' \subset \val{X}_X$ and $V \subset \val{X}_{\Pa(Y)-\{X\}}$
such that $U\times V$ and $U'\times V$ are measurable (with respect to $Q$),
and $E[f_Y|\pa \in U\times V] \neq E[f_Y|\pa \in U'\times V]$, so
one can think of Def.\ \ref{def:obs_graph} using this notion instead.
Because measure-theoretic intricacies of the problem do not seem to
aid the understanding of the main contents of this paper,
we do not detail these problems in the main text.

\subsection{Finite-Sample Generalizations}\label{apdx:FiniteSample}

In practice, when only a finite number of samples is available,
the distinctions (descriptive vs.\ physical changes) discussed in this paper
also occur for reasons different from non-overlapping supports (of observations and
mechanisms): Statistical power of independence tests often relies for example
on sufficient width (compared to first derivative of the mechanism and noise on the target)
of the observational distribution of the source.
More generally, the specific choice of independence-test matters.
In this section, we outline how our results generalize to the finite-sample case,
how analogues of the previously introduced graphical objects lead to a very similar
abstract structure, and why finite-sample properties are even more difficult: There
is a "gap" (similar to §\ref{apdx:CSI})
between never detectable (with probability less than a small $p_0$ detectable)
and confidently detectable (with probability larger $1-\epsilon$ detectable)
that does not occur in the asymptotic case. 

One may replace the definition \ref{def:obs_graph} of $G[\mathcal{F},Q]$ by the following harder
to formalize, but for some problems more practical idea:
For an estimator $\hat{d}$ of a dependence-measure $d$,
let $G[\mathcal{F}, Q, \hat{d}, N, p_0, \epsilon]$ be the graph defined
via by parent-sets with $X \in Pa(Y)$ if,
fixing a sample-count $N$ and error-rate $p_0$, the estimator $\hat{d}$
has enough (up to $\epsilon$) statistical power to find dependence
in the sense of
$\exists d_0$: $Q(\hat{d} \geq d_0) > 1-\epsilon$ -- with $P_{\txt{null}}(\hat{d}\geq d_0)<p_0$ in the product\Slash{}independent null-distribution --
where $Q(\hat{d})$ is the distribution of $\hat{d}$ evaluated on
$(v_X, f_Y(v))$ on $N$ samples $v$ drawn from $Q(V)$.
See §\ref{apdx:TransferForPhys}.
This does not seem to change the abstract structure (kinds of graphs and their relationships), except that
an additional "gap" similar to §\ref{apdx:CSI} opens, because there are edges with effect-sizes
that are detectable with probability between $p_0$ and $1-\epsilon$.

This captures not only the reality of what we see (the observational support), but also
the reality of how we see it (the dependence-test).
In practice the result of a causal discovery algorithm does depend on the
independence test used, so this describes what
is identifiable from data.
Its interpretation in terms of causal inference (\eg effect estimation)
is harder, but this is not a failure of the approach,
but rather a "real" problem:
Given \eg an SCM with linear effects and Gaussian noise-terms,
such that all (non-trivial) effects are large enough for
a suitable to this data test (\eg partial correlation)
to have power $1-\epsilon$, then the discovered graph is valid for
effect estimation
(up to error-rates bounded by $p_0$ and $\epsilon$ corrected for multiple-testing,
we have $G[\mathcal{F}, Q, \hat{d}, N, p_0, \epsilon] = G^{\txt{xyz}}[M]$,
where "xyz" stands for a graph corresponding to a specific choice of $Q$, which
will also has implications for $N$).
If the data is not suitable to the used test in this sense,
we still discover $G[\mathcal{F}, Q, \hat{d}, N, p_0, \epsilon]$,
but it is no longer trivially suitable for effect estimation
(but \eg a correlation-based test might still capture causal effect mean-values,
even though no longer higher moments).
We leave this general problem to future research,
but it seem interesting that statistically precise statements about
validity of certain types of effect-estimations appear to be formally 
accessible. For counter-factual properties, one additionally to $\hat{d}$
needs an estimator for conditional densities.

The choice of independence test seems to usually be
seen as governed by properties of available data
(which is even in theory only possible to a certain
degree \citep{shah2020hardness}),
our point here is that there is an associated
graphical object, whose practical usefulness 
depends on the application additional to the data.

\section{Solvability and Solution-Functions}\label{sec:solvability}

Our graphical objects no longer have a simple connection to an set of mechanisms alone,
rather they depend on observational support.
This means many of the usual proof-techniques (most notably path-blocking)
have no evident analogue when discovering these structures from data.
A systematic treatment of "Markov"-properties needs a different approach.
We show that the problem can be studied via properties of solution-functions,
hence we briefly study solvability of models.

Using only "context insensitive" independence-tests on the "pooled"
data, fails to be Markov to the visible graph (some links cannot be detected
as absent -- actually exactly those links in the acyclification \citep{BongersCyclic}.

Some acyclicity-property \emph{is} needed also with CSI.
An easy to visualize property is the  following "strong" regime-acyclicity
(but we often only require the slightly weaker "solvable for $R$ and
weakly regime-acyclic", see lemma\,\ref{lemma:strong_acyclic_and_solvable}):
\begin{definition}\label{def:acyclicity}
	We call a SCM $M$ weakly ($R$-)regime-acyclic, if
	$\forall r$, the regime-graph $\gDescr{G}\FixedR[M]$ is acyclic. 
	
	We call a model $M$ strongly ($R$-)regime-acyclic, if
	it is weakly ($R$-)regime-acyclic and
	no cycle in $\gUnion{G}[M]$ involves any union-ancestor of $R$ (including $R$ itself).
\end{definition}

Easily usable models are typically "solvable" as systems of equations from the noise-terms
(this is a notion often employed \eg to study counterfactuals \citep{PearlBook} and has been used to
study cyclic models \eg in \citep{BongersCyclic}, see §\ref{apdx:ConnectionsToLit}):

\begin{definition}
	A set of mechanisms $\mathcal{F}$ is (uniquely) solvable for $X_i$,
	on $\Omega \subset \val{N}$
	if there is a (unique)  mapping $F_i : \Omega \rightarrow \val{X}_i$ such that
	$X_i = F_i( \eta_1, \ldots, \eta_N )$.
	
	$\mathcal{F}$ is (uniquely) solvable on $\Omega \subset \val{N}$,
	if for all $i$ it is (uniquely) solvable for $X_i$.
	
	A model $M$ is (uniquely) solvable (for $X_i$),
	if its mechanisms $\mathcal{F}$
	are (uniquely) solvable (for $X_i$) on $\supp(P_\eta)$.
\end{definition}

We would expect such models to have "good" solution properties. There is a small
caveat however: Our graph-definitions (and hence acyclicity-definitions) require
a "weak" solvability, namely the observable distribution $P_{\mathcal{F},P_\eta}(V)$
has to exist (with unique support). In practice, when given observations -- presumably from
an SCM -- than this SCM is evidently "weakly solvable" in this sense.
Here, "weakly solvable" in turn implies (unique) solvability in the intuitive sense.

\begin{lemma}\label{lemma:strong_acyclic_and_solvable}
	Let $M$ be weakly regime-acyclic and the observable distribution $P_{\mathcal{F},P_\eta}(V)$
	exists.
	Then:
	\begin{align*}
	\txt{$M$ is strongly regime-acyclic}
	\quad&\Rightarrow\quad
	\txt{$M$ is uniquely solvable for $R$}\\
	\txt{$M$ is uniquely solvable}
	\quad&\Leftrightarrow\quad
	\txt{$M$ is uniquely solvable for $R$}
	\end{align*}
\end{lemma}
\begin{proof}
	It is well-know that acyclic SCMs are solvable.
	The idea is simply as follows: Let $l(i)$ be the length of the longest
	incoming path to $X_i$, \ie the count of ancestors in a path
	$\gamma = [A_1 \rightarrow A_2 \rightarrow \ldots \rightarrow X_i ]$
	with all arrows pointing towards $X_i$.
	Then inductively (over $l$) show $M$ is solvable for all $X_i$ with $l(i) = l$.
	The inductive start $l=0$ is trivial, as nodes with $l(X_i)$ are roots (\ie do not have
	parents), so $l(i) = 0$ $\Rightarrow$ $f_i = f_i(\eta_i)$, thus the solution $F_i = f_i$ works.
	For the inductive step, note that $l(i) = l+1$ $\Rightarrow$ $l(\Pa_i) \leq l$,
	thus have solution functions $F_{\Pa_i}$, the solution $F_i = f_i(F_{\Pa_i}, \eta_i)$
	works for $X_i$.
	
	Let $M$ be strongly regime-acyclic.
	There are no cycles involving ancestors (in $\gUnion{G}[M]$ of $R$ (including $R$).
	Thus the above inductive argument works restricted to ancestors of $R$ (including $R$),
	because parents of ancestors of $R$ are also ancestors of $R$
	and within the support $\Omega = \supp(P_\eta)$ we only need union-parents.
	Therefore the model is solvable for ancestors (in $\gUnion{G}[M]$) of $R$ (including $R$).
	
	Next, knowing $F_R$, we can "split" the space of noise-values into the disjoint union
	$N = \coprod_{r} F_R^{-1}(\{r\})$
	and note that for $\vec{\eta} \in F_R^{-1}(\{r\})$ we know $R=F_R(\vec{\eta})=r$.
	Knowing $R=r$, each node depends (for these $\vec{\eta}$) at most on its parents in the respective $\gDescr{G}\FixedR[M]$
	(by definition of $\gDescr{G}\FixedR[M]$).	
	Hence we can repeat the argument above on the \emph{acyclic} $\gDescr{G}\FixedR[M]$
	to find $X_i = F^{R=r}_i(\vec{\eta})$ for $\vec{\eta} \in F_R^{-1}(\{r\})$
	(this $F^{R=r}_i$ is of course the same one as in definition\,\ref{def:F_Z} below,
	as is immediate for the definition of $F_i$ in the next paragraph).
	
	Define $F_i := F_i^{R=F_R(\vec{\eta})}(\vec{\eta})$. By disjointness of the
	$F_R^{-1}(\{r\})$ this is well-defined, because every $\vec{\eta}$ is mapped to some $r$
	by $F_R$ it is defined everywhere.
	
	Finally the backwards direction 
	$M$ is solvable for $R$	$\Rightarrow$ $M$ is solvable is trivial.
\end{proof}

For solvable models (with almost everywhere continuous densities),
conditioning can be understood as restriction of the sample-space: 

\begin{definition}\label{def:F_Z}
	If $M$ is solvable, define,
	\begin{equation*}
	F_i^{Z=z} := F_i|_{F_Z^{-1}(\{z\})} : F_Z^{-1}(\{z\}) \rightarrow \val{X}_i
	\end{equation*}
	(we allow $Z$ to be multivariate).
\end{definition}

\begin{cor}\label{cor:solution_depends_on_regime_anc_only}
	Given a solvable, weakly regime-acyclic model,
	then, for an arbitrary variable $X$:
	\begin{enumerate}[label=(\alph*)]
		\item		
		$F_X$ depends only on noise-terms of ancestors of $X$ in $\gUnion{G}[M]$,
		\ie is constant in all other noise-terms an can thus be written as a function
		of ancestors' noise-terms only.
		\item
		$F^{R=r}_X$ depends only  on noise-terms of
		ancestors of $X$ in $\gDescr{G}\FixedR[M]$.
	\end{enumerate}
\end{cor}
\begin{proof}
	This is apparent from the proof of lemma\,\ref{lemma:strong_acyclic_and_solvable}:
	
	$F_i$ was constructed inductively from parents and their noise,
	and from parents of parents and their noise etc.\ (in $\gUnion{G}[M]$)
	thus from noises of
	ancestors in $\gUnion{G}[M]$ (with roots depending only on their own noise).
	
	$F^{R=r}_i$ was constructed in the same way from noises of ancestors in
	$\gDescr{G}\FixedR[M]$.
\end{proof}

Note that corollary \ref{cor:solution_depends_on_regime_anc_only}
encodes information about support and parental relations on a given support.
We use this knowledge to replace path-blocking arguments for obtaining
a "Markov"-property.

\section{Faithfulness}\label{apdx:Faithfulness}

There are multiple ways in which faithfulness can fail to hold:
Finetuning (cancelations) between paths might be the most discussed one,
but also deterministic relations between variables lead to non-unique parent-sets
and thus non-well-defined graphs.
But also regime-specific changes of mechanism (as for $Y=\mathbbm{1}(R) \times X + \eta_Y$)
can be understood as a faithfulness violation (the intervened model
$\mathcal{F}_{\PearlDo(R=r)}$ is not faithful to $G[\mathcal{F}]$),
as has also been observed \eg by \citep{JCI}.

One may thus take a more general perspective: We can think of faithfulness as
an assumption "bridging" the gap between observations and a graphical object
associated to the model. The "width" of this gap depends on what aspects of the above
mentioned problems are encoded in this graphical object!
\Eg for a regime-specific change of mechanism (as above), instead of
saying "$\mathcal{F}_{\PearlDo(R=r)}$ is not faithful to $G[\mathcal{F}]$"
and giving up, we clearly want to \emph{learn and understand} a "regime-specific graph",
which captures the difference and for which the context-specific independence is "expected"
rather than a violation of assumptions.

The additional inclusion of the support into the definition of the graphical object is,
from this perspective, just the logical next step. For example looking at the discussion
around the definition \ref{def:visible} of the "visible" graph, the reader will notice,
that we moved the support-related aspects of faithfulness into the graph,
while all other aspects (including minimality of the parent-sets) are left in the "gap"
that is bridged by assuming "$M$ is faithful to $\gVisible{G}[M]$".

Clearly the abstract argument
is in no way specific to support aspects of faithfulness, similarly one could
\eg weaken determinism-assumptions encapsulated in the faithfulness assumption by changing
the graphical objects etc., however, a thorough and systematic treatment of
faithfulness from this perspective turned out to be quite complex, so we will
leave this issue to future research for now.

Another faithfulness-related problem is discussed in §\ref{apdx:CSI}.

\subsection{Justification of Assumptions in the Main Text}

We briefly repeat the argument given in \citep{MethodPaper},
to justify assumption \ref{ass:faithful}.

Generally, a probability distribution $P$ is faithful to a DAG $G$
if independence $X \independent_P Y | Z$ with respect to $P$
implies d-separation $X \independent_G Y | Z$ with respect to $G$.
As discussed in \citep{MethodPaper},
this means if $G' \subset G$ is (strictly) sparser,
then faithfulness to $G'$ is (strictly) weaker than faithfulness to $G$.
Now, $\gDescr{G}_{R=r}\subset \gUnion{G}=\gVisible{G}$, so
"$P_M(\ldots)$ is faithful to $\gDescr{G}_{R=r}$" is weaker than
the standard assumption "$P_M(\ldots)$ is faithful to $\gVisible{G}$",
and similarly (excluding links involving $R$),
$\gDescr{\bar{G}}_{R=r}$ is sparser than what one would expect for a
"graph of the conditional model" (there is no selection-bias induced edges in
$\gDescr{\bar{G}}_{R=r}$) so "$P_M(\ldots|R=r)$ is faithful to $\gDescr{\bar{G}}_{R=r}$"
is also weaker than what one would expect to assume.
One can thus give an adjacency-faithfulness result that
essentially corresponds to standard-assumptions as explained above:

\begin{lemma}\label{lemma:faithful_is_faithful}
	Given $r$, assume both $P_M$ is faithful to $\gDescr{G}_{R=r}$
	and $P_M(\ldots|R=r)$ is faithful to $\gDescr{\bar{G}}_{R=r}$
	(we will refer to this condition as $r$-faithfulness,
	or $R$-faithfulness if it holds for all $r$).
	Then:
	\begin{equation*}
	\exists Z \txt{ s.\,t.\ }
	\left\{\quad\begin{aligned}
	X &\independent Y | Z \quad\txt{or}\\
	X, Y \neq R \txt{ and }X &\independent Y | Z, R=r
	\end{aligned}\quad\right\}
	\quad\Rightarrow\quad
	\txt{$X$ and $Y$ are not adjacent in $\gDescr{G}_{R=r}$}
	\end{equation*}
\end{lemma}

\begin{proof}
	The statement is symmetric under exchange of $X$ and $Y$, so
	it is enough to show $X \notin \gDescr\Pa_{R=r}(Y)$. 
	We do so by contradiction: Assume $X \in \gDescr\Pa_{R=r}(Y)$ and let $Z$ be arbitrary.
	$Z$ can never block the direct path $X \rightarrow Y$,
	so they are never d-separated $X \dependent_{\gDescr{G}_{R=r}} Y | Z$.
	By (the contra-position of) the faithfulness assumptions,
	thus $X \dependent_P Y | Z$ and if $X, Y \neq R$ also $X \dependent_P Y | Z, R=r$
	(the second statement is by definition
	the same as $X \dependent_{P(\ldots|R=r)} Y | Z$).
\end{proof}

\subsection{An Example that is not Strongly Faithful}\label{apdx:not_strongly_faithful}

Below are an example and discussion to shed some light on why the union-property (lemma \ref{lemma:union_is_union})
required an additional faithfulness assumption.

\begin{example}
	Not strongly $R$-faithful:\\
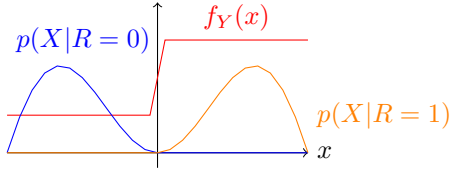
	\begin{minipage}{0.5\textwidth}
		\begin{tikzpicture}[domain=-2:2]
			\draw[->] (-2,0) -- (2,0) node[right] {$x$};
			\draw[->] (0,-0.2) -- (0,2);
			
			\draw[color=red]    (-2,0.5) -- (-0.1,0.5) -- (0.1,1.5) -- node[pos=0.5,above]{$f_Y(x)$} (2, 1.5);
			\draw[color=blue]   plot (\x,{0.5*(\x - sqrt(\x*\x))*sin(90*\x)});
			\draw[color=blue]	(-2, 1.2)node[anchor=south west] {$p(X|R=0)$};
			\draw[color=orange]   plot (\x,{0.5*(\x + sqrt(\x*\x))*sin(90*\x)});
			\draw[color=orange]	(2, 0.5)node[anchor=west] {$p(X|R=1)$};
		\end{tikzpicture}
	\end{minipage}
	\begin{minipage}{0.45\textwidth}
		For the functional relationships on the left, $Y$ is a function of $X$
		and $X \in \gUnion\Pa(Y)$, but
		$X \notin \gDescr\Pa\FixedR(Y)$ for both $r=0$ and $r=1$.
	\end{minipage}
\end{example}
This is a non-determinism issue (we could write $f_Y$ as a function of $R$ only in
the observational support of the union),
and is supposed to be excluded by faithfulness (of the union-model).
There should be $Z = \gUnion\Pa(Y)$ with $X \dependent_{\gUnion{G}} Y | Z, R$
(because the direct path cannot be blocked), but $X \independent Y | Z, R$
(because of the deterministic relation $R$ explains away $X$).
For cyclic models there is a subtle problem however:
If $Y$ is part of a directed cycle where $X$ is a parent of another node $Z$ in that cycle,
then possibly $X \dependent Y | Z, R$, \ie faithfulness may not be violated (formally),
because there is a link in the acyclification \citep{BongersCyclic} that "saves" us.

The problem formally also reveals itself as follows:
Faithfulness of the union-model implies that for every $Z$ (again because the direct
path cannot be d- or $\sigma$-blocked) $X \dependent Y | Z, R$,
which is equivalent (as can be seen \eg by the characterization of independence as
factorization of the joint) to $\exists r$ with $X \dependent Y | Z, R=r$,
which suggests that there is a context with this link.
But there could \eg be $Z \neq Z'$ with $X \dependent Y | Z, R=0$ and
$X \dependent Y | Z', R=1$, which in the cyclic case can (non-trivially) happen
by union-parents potentially not being valid separating-sets.

This cannot easily be solved by a minimality-condition
\citep[Def.\ 2.6]{BongersCyclic}
on parents either: In the example above both possible parent-sets of $Y$,
which are $\{X\}$ or $\{R\}$ are of cardinality $1$ so no unique minimal parent-set
exists, and \eg the choice via \citep[Def.\ 2.6]{BongersCyclic}
is not well-defined (which is not a problem, because normally
a suitable faithfulness assumption excludes deterministic relation-
ships; this is really a determinism issue, not a minimality issue).

\section{Details on Connections to JCI- and Transfer-Arguments}\label{apdx:ConnJCI}

This section contains proofs of the statements in §\ref{sec:ConnectToJCI} and examples.

\subsection{Inferring the Union-Graph}

Recall from remark \ref{rmk:EdgesFromR},
that edges from $R$ into directed union-cycles containing a child of $R$ cannot be
deleted by our independences. We will hence mostly focus on edges elsewhere in the graph,
using the "barred" notation ($\gDescr{\bar{G}}\FixedR$ etc.).
Generally, a causal model is only Markov to the acyclification
of its visible ("standard") graph $\Acycl(\gVisible{G}[M])$
while, for strongly regime-acyclic models we here have:

\theoremstyle{plain}
\newtheorem*{lemma_union_id}{Lemma \ref{lemma:UnionIdentifiable}}
\begin{lemma_union_id}
	Let $M$ be a strongly $R$-regime-acyclic, strongly $R$-faithful, causally sufficient model, then
	\begin{equation*}
	\gVisible{\bar{G}}[M] = \gUnion{\bar{G}}[M]
	= \cup_r \gDetect{\bar{G}}\FixedR[M]
	\end{equation*}
	is identifiable away from $R$ by ($R$-context-specific) independences.
\end{lemma_union_id}
\begin{proof}
	By lemma \ref{lemma:union_is_union},
	$\gUnion{G} = \cup_r \gDescr{G}\FixedR$,
	thus 
	(a) $\gUnion{\bar{G}} = \cup_r \gDescr{\bar{G}}\FixedR$.
	While $\gDetect{G}\FixedR \neq \gDescr{G}\FixedR$ in general,
	by prop.\ \ref{prop:markov} and ass.\ \ref{ass:faithful} (see §\ref{apdx:CSI}),
	$\gDescr{G}\FixedR \subset \gDetect{G}\FixedR \subset \gIdent{G}\FixedR$ thus
	(b) $\gDescr{\bar{G}}\FixedR \subset \gDetect{\bar{G}}\FixedR \subset \gIdent{\bar{G}}\FixedR$.
	
	Combining (a) with (b), thus
	\begin{equation*}
		\cup_r \gDetect{\bar{G}}\FixedR \overset{(b)}{\supset}
		\cup_r \gDescr{\bar{G}}\FixedR \overset{(a)}{=} \gUnion{\bar{G}}
		\txt.
	\end{equation*}
	On the other hand, by lemma \ref{lemma:no_phys_anc_anc},
	$\gIdent{G}\FixedR \subset \gUnion{G}$ and thus
	(c) $\gIdent{\bar{G}}\FixedR \subset \gUnion{\bar{G}}$,
	so that
	\begin{equation*}
		\cup_r \gDetect{\bar{G}}\FixedR
		\overset{(d)}{\subset}
		\cup_r \gIdent{\bar{G}}\FixedR
		\overset{(b)}{\subset}
		\gUnion{\bar{G}}\txt.
	\end{equation*}
\end{proof}

\subsection{Interring the Transfer-Graph}

\theoremstyle{plain}
\newtheorem*{lemma_phys_r1}{Lemma \ref{lemma:R1}}
\begin{lemma_phys_r1}
	If $R \notin \gUnion{\Anc}(Y)$, then
	$\gTransf\Pa_{R=r}(Y) = \gUnion\Pa(Y)$,
	\ie the change is non-physical (by observational non-accessibility).
\end{lemma_phys_r1}
\begin{proof}
	This follows directly from lemma \ref{lemma:PhysChangesAreRegimeChildren}.
\end{proof}
\newtheorem*{cor_phys_r1}{Cor.\ \ref{cor:R1}}
\begin{cor_phys_r1}
	If $R \notin \gUnionId{\Anc}(Y)$, then
	$\gTransf\Pa_{R=r}(Y) = \gUnionId\Pa(Y)$.
\end{cor_phys_r1}
\begin{proof}
	This follows directly from lemma \ref{lemma:R1} and rmk.\ \ref{rmk:EdgesFromR}
	(see also lemma \ref{lemma:g_detect}).
\end{proof}

If $R$ (or conditioning on $R$) does not change the distribution of
ancestors, no state-induced effects occur:

\newtheorem*{lemma_phys_r2}{Lemma \ref{lemma:R2}}
\begin{lemma_phys_r2}
	Assuming strong regime-acyclicity.
	If $X \in \gUnion\Pa( Y ) - \gIdent\Pa\FixedR( Y )$
	and $R \in \gUnion\Pa( Y )$,
	and	$\gUnion\Anc(R) \cap \gUnion\Anc(\gUnion\Pa(Y)-\{R\}) = \emptyset$,
	then
	$X \notin \gTransf\Pa( Y )$ (\ie the change is "physical" not
	just by state).
\end{lemma_phys_r2}

\begin{proof}
	By lemma \ref{lemma:split_P_eta}, the noise-terms of
	nodes in $\gUnion\Anc(Y)$ are unchanged by conditioning on $R$
	\ie $P(\eta_{\gUnion\Anc(Y)}|R) = P(\eta_{\gUnion\Anc(Y)})$
	and by corollary \ref{cor:solution_depends_on_regime_anc_only}a
	applied to $R\neq W \in \gUnion\Pa(Y)$ shows $W = F_W(\eta_{\gUnion\Anc(W)})$,
	with $\gUnion\Anc(W) \subset \gUnion\Anc(Y)$ thus
	$P(X_{\gUnion\Pa(Y)-\{R\}}|R) = P(X_{\gUnion\Pa(Y)-\{R\}})$.
	Therefore the support on parents did not change
	and the change must be physical.
\end{proof}
\newtheorem*{cor_phys_r2}{Cor.\ \ref{cor:R2}}
\begin{cor_phys_r2}
	Assuming strong regime-acyclicity.
	If $R\neq X \in \gUnionId\Pa( Y ) - \gIdent\Pa\FixedR( Y )$
	and $R \in \gUnionId\Pa( Y )$,
	and	$\gUnionId\Anc(R) \cap \gUnionId\Anc(\gUnionId\Pa(Y)-\{R\}) = \emptyset$,
	then
	\begin{enumerate}[label=(\alph*)]
		\item
		there is a link into the strongly connected component of $Y$
		that vanishes in $\gTransf{G}$, but not in $\gUnionId{G}$,
		\ie there is a physical change.
		\item
		if $Y$ is not part of a directed union-cycle, then
		$X \notin \gTransf\Pa( Y )$,
		\ie there is a physical change of this particular link.
	\end{enumerate}
\end{cor_phys_r2}
\begin{proof}
	Excluding $R$, $X \in \gUnionId\Pa( Y ) \Rightarrow  X \in \gUnion\Pa( Y )$.
	Similarly both $\gUnionId\Anc(R)$ and $\gUnionId\Anc(\gUnionId\Pa(Y)-\{R\})$
	exclude $R$, so we can replace them by $\gUnion\Anc$.
	Since $\gUnion{G} \subset \gUnionId{G}$,
	also $R \in \gUnionId\Pa( Y ) \Rightarrow R \in \gUnion\Pa( Y )$.
	
	Thus the lemma applies.
	the vanishing link starts at $X \neq R$ (thus is away from $R$)
	and ends at an element of the strongly-connected component of $Y$.
	If $Y$ is not part of a directed cycle, the strongly-connected component of
	$Y$ is simply $\{Y\}$, and there is only a unique choice.
\end{proof}

\subsection{Validity of Transfer}\label{apdx:TransferForPhys}

One can also use a transfer-argument to construct a test which
deletes edges from the union-graph only if there is evidence that
the mechanism did in fact change. See also §\ref{apdx:FiniteSample}.

Fix dependency measure $d$ and estimator $\hat{d}$.
Assume, using $\hat{d}$ (and some null-distribution and $p$-value threshold),
we found a link $X \rightarrow Y$ with identifiable (\eg by adjusting for $Z$)
controlled direct effect of $X$ on $Y$
and such that this link vanishes in one context $r_0$.
We want to distinguish between:
\begin{itemize}
	\item
	The nullhypothesis: The change in $P(X)$ suffices to explain
	the failure to reject independence on finite-data.
	\item
	The alternative: The mechanism (or the noise on $Y$) have changed.
\end{itemize}

On the $\hat{d}$-dependent context,
learn an estimator $\hat{P}_X$ of $P(X,Z|R=r_0)$ and $\hat{P}_{Y|X}$ of $P(Y|X,Z)$
(\ie of the kernel $x \mapsto f_Y(x,-)_* \eta_Y$
containing the observable information about $f_Y$ and $\eta_Y$)\footnote{Under the null-hypothesis,
	learning $g$ from the pooled data is ok, so even though in the alternative hypothesis $g$ changes,
	for rejecting the null, learning $g$ from the pooled data is fine, even though learning
	from a single or all other contexts might improve power.}
by some conditional-density learning method.
For a total of $K$ datasets of size $N$ each, draw
$((x_1, z_1), \ldots, (x_N, z_N))$ from $\hat{P}_X$,
then draw $y_i$ from $\hat{P}_{Y|X}(Y|X=x_i,Z=z_i)$.
On these datasets, generate dependence-measures (or test for independence)
using $\hat{d}$ leading to a distribution $\hat{P_d}$.
If the result for $\hat{d}$ on the original data in the $\hat{d}$-independent regime
is plausible under $\hat{P_d}$ (or the test results on the $K$ many datasets
are $1-\alpha$ often "independent"), then the changed support of $X$ is
sufficient to explain the "independence" (or rather the failure of $\hat{d}$ to
detect any dependence) in this regime
-- assuming $\hat{P}_{Y|X}$ approximates the true $P(Y|X,Z)$ sufficiently well
(see below).
Otherwise we can reject the null-hypothesis that the change in support of $X$ alone could
explain the absence of this link.

The reliance on sufficiently fast convergence of $\hat{P}_{Y|X}$ is
conceptually similar to the convergence of regressors
in conditional independence testing with regressing out.
\Ie when using a parametric model, for evaluating p-values, one 
has to take into account the additional number of degrees of freedom,
for non-parametric models, \eg bootstrapping approaches could be used.
We acknowledge that this is in practice a very difficult problem.
We leave it to future research, our present intent is to illustrate,
that this seems -- in principle -- to also be a testable hypothesis.

\subsection{Limiting (Extreme) Cases}

The following "extreme" case is formally trivial,
but provides some insights:

\begin{example}\label{example:p1}
	Given $P(R=r_0) = 1$ (which we typically exclude by
	the way we define regime-indicators, but which we can
	think of as a limiting case in practice),
	we observe:
	$P(\ldots|R=r_0)=P(\ldots)$, 
	so also the supports agree and
	$\gUnion{G} = \gDescr{G}_{R=r_0} = \gPhys{G}_{R=r_0}$.
	I.\,e., in this case our results collapse to the standard
	results for $\gUnion{G}$.
\end{example}

From the perspective that, for the single-context
case, the question about what is happening outside
the support should probably be considered purely philosophical,
this is a good sign: If our objects capture empirically accessible
information, then they should not make claims about the single-context case.

\end{document}